%% file: main.tex
\documentclass[journal]{IEEEtran}

\usepackage{multicol}

\usepackage{amsthm,amsmath,amssymb}
\usepackage{booktabs}
\usepackage{graphicx}
\usepackage{subcaption}
\usepackage[linesnumbered,ruled,noend]{algorithm2e}
\usepackage{algorithmic}
\usepackage{svg}
\usepackage{xspace}

\usepackage{xcolor}

\usepackage{cite}
\usepackage{hyperref}
\usepackage{multirow}

\newtheorem{theorem}{Theorem}
\newtheorem{problem}{Problem}

\newtheorem{assumption}{Assumption}

\newtheorem{corollary}{Corollary}
\newtheorem{lemma}{Lemma}
\newtheorem{definition}{Definition}
\newtheorem{remark}{Remark}

\newcommand{\rv}[1]{{#1}}

\newcommand{\aoalg}{\textit{Kino-PAX$^{+}$}\xspace}
\newcommand{\alg}{\textit{Kino-PAX}\xspace}

\newcommand{\free}{\text{valid}}
\newcommand{\goal}{\text{goal}}
\newcommand{\init}{\text{init}}

\newcommand{\cost}{\mathrm{cost}}
\newcommand{\traj}{\mathbf{x}}
\newcommand{\Traj}{\mathbf{X}}

\newif\ifarxiv
\arxivtrue      

\begin{document}

\title{Kino-PAX$^{+}$: Near-Optimal Massively Parallel Kinodynamic Sampling-based Motion Planner
}

\author{Anonymous Authors}

\author{Nicolas Perrault$^{1}$, Qi Heng Ho$^{2}$, and Morteza Lahijanian$^{1}$%
\thanks{$^{1}$The authors are with the Department of Aerospace Engineering Sciences, University of Colorado Boulder, CO, USA. 
        {\tt\small \{nicolas.perrault, morteza.lahijanian\}@colorado.edu}}%
\thanks{$^{2}$Q. H. Ho is with the Department of Aerospace and Ocean Engineering, Virginia Tech, VA, USA. 
        {\tt\small qihengho@vt.edu}}%
}

\maketitle

\begin{abstract}
    Sampling-based motion planners (SBMPs) are widely used for robot motion planning with complex kinodynamic constraints in high-dimensional spaces, yet their serial computation design results in planning speeds that scale poorly with problem complexity. Recent efforts to parallelize SBMPs have achieved significant speedups in finding feasible solutions; however, they provide no guarantees of optimizing an objective function. We introduce Kino-PAX$^{+}$, a massively parallel kinodynamic SBMP with asymptotic near-optimal guarantees.  Kino-PAX$^{+}$ builds a sparse tree of dynamically feasible trajectories by decomposing traditionally serial operations into three massively parallel subroutines. The algorithm focuses computation on the most promising nodes within local neighborhoods for propagation and refinement, enabling rapid improvement of solution cost. We prove that, while maintaining probabilistic $\delta$-robust completeness, this focus on promising nodes ensures asymptotic $\delta$-robust near-optimality. Our results show that Kino-PAX$^{+}$ finds solutions up to three orders of magnitude faster than existing serial methods and achieves lower solution costs than a state-of-the-art GPU-based planner.
\end{abstract}


\input macros.tex
 \section{Introduction}
    \label{sec:intro}
    \input{sections/introduction}

\subsection{Related Work}
    \label{sec:related}
    \input{sections/related}


\section{Problem Formulation}
    \label{sec:problem}
    \input{sections/problem}

    
\section{\aoalg Algorithm}
    \label{sec:alg}

\input{sections/Algorithm}

\section{Analysis}
    \label{sec:Analyis}
    \input{sections/Analysis}



\section{Empirical Evaluation}
    \label{sec:OKPAX_experiments}
    \input{sections/experiments}

\section{Conclusion}
    \label{sec:conclusion}
    \input{sections/conclusion}

\appendix
\input{sections/GPU_implementation}

\bibliographystyle{IEEEtran}
\bibliography{refs,references}



\end{document}

%% file: macros.tex
\newcommand{\diff}[2]{\frac{\partial #1}{\partial #2}}
\newcommand{\diffr}[1]{\diff{#1}{r}}
\newcommand{\diffth}[1]{\diff{#1}{\theta}}
\newcommand{\diffz}[1]{\diff{#1}{z}}

\newcommand{\vth}{V_{\theta}}

\newcommand{\twochoices}[2]{\left\{ \begin{array}{lcc}
        \displaystyle #1 \\ \vspace{-10pt} \\
        \displaystyle #2 \end{array} \right. } 

\newcommand{\threechoices}[3]{\left\{ \begin{array}{lcc}
        #1 \\ #2 \\ #3 \end{array} \right. }    

\newcommand{\fourchoices}[4]{\left\{ \begin{array}{lcc}
        #1 \\ #2 \\ #3 \\ #4 \end{array} \right. }      

\newcommand{\twovec}[2]{\left(\begin{array}{c} #1 \\ #2 \end{array}\right)}
\newcommand{\threevec}[3]{\left(\begin{array}{c} #1 \\ #2 \\ #3 \end{array}\right)}
\newcommand{\twomatrix}[4]{\left(\begin{array}{cc} #1 & #2 \\ #3 & #4 \end{array}\right)}

%% file: sections/introduction.tex
\IEEEPARstart
{A}{utonomous} robotic systems deployed in dynamic environments require fast, reactive motion planning that accounts for complex kinematics and dynamics. 
However, feasibility alone is insufficient for high-performance operation; such systems demand \emph{high-quality} trajectories that minimize costs such as path length, energy consumption, execution time, or control effort.
Despite recent progress in accelerating fast kinodynamic motion planning via, e.g., parallelization, achieving \emph{fast optimal} planning remains a significant challenge.
In this work, we aim to enable low-latency, near-optimal motion planning for complex and high-dimensional kinodynamical systems by exploiting the parallel architecture of GPU-like devices.

\textit{Sampling-based motion planners} (SBMPs) have emerged as the dominant approach for these challenges, successfully handling complex dynamics 
\cite{lavalle2001randomized, phillips2004guided, plaku2010motion, ladd2005fast,  sucan2011sampling}, complex tasks \cite{bhatia2010sampling, Maly:HSCC:2013, Lahijanian:TRO:2016}, and stochastic environments 
\cite{luders2010chance, Luna:AAAI:2014, Luna:WAFR:2015, Ho:ICRA:2023}. These methods have been extended to optimize for a cost function, achieving asympotic near-optimality \cite{li2016asymptotically, Hauser2016aox, ho:2022:ICRA, karaman2011rrtstar}. Nevertheless, traditional SBMPs are typically designed for serial computation, limiting their performance to the clock speed of a single CPU core. While recent methods can find solutions within seconds for simple systems and tens of seconds for complex ones, these latencies remain a bottleneck for reactive deployment. With the stagnation of single-thread performance improvements, parallel architectures like GPUs offer the only viable path to substantial speedups. Yet, standard asymptotically optimal algorithms are inherently sequential, making them inefficient when naively parallelized. 

\begin{figure}[t]
    \centering
     \begin{subfigure}[b]{0.32\columnwidth}
        \centering
        \includegraphics[width=\textwidth]{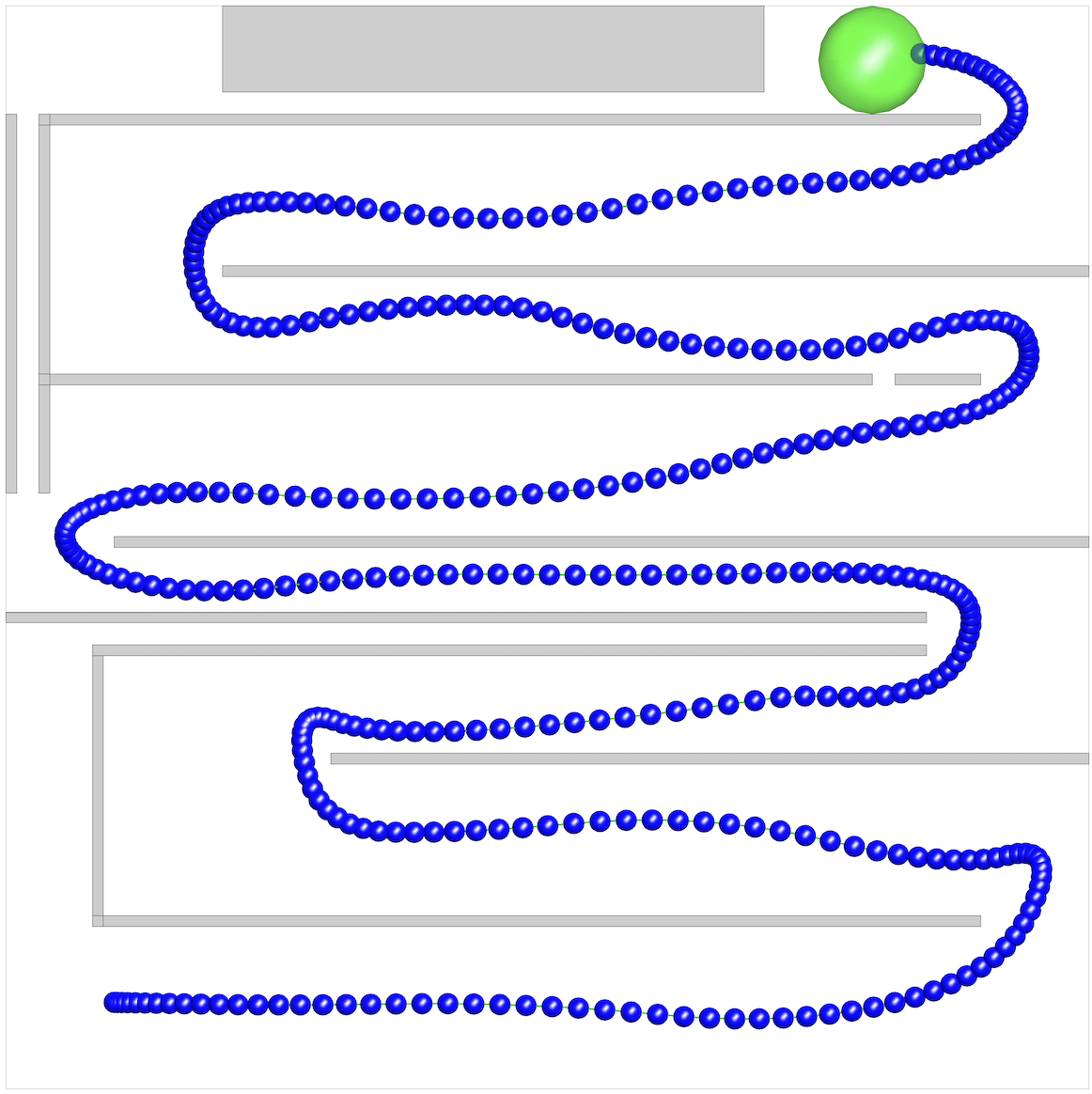}
        \caption{First Solution}
        \label{fig:AOKPAX_first_sol2}
    \end{subfigure}
    \begin{subfigure}[b]{0.32\columnwidth}
        \centering
        \includegraphics[width=\textwidth]{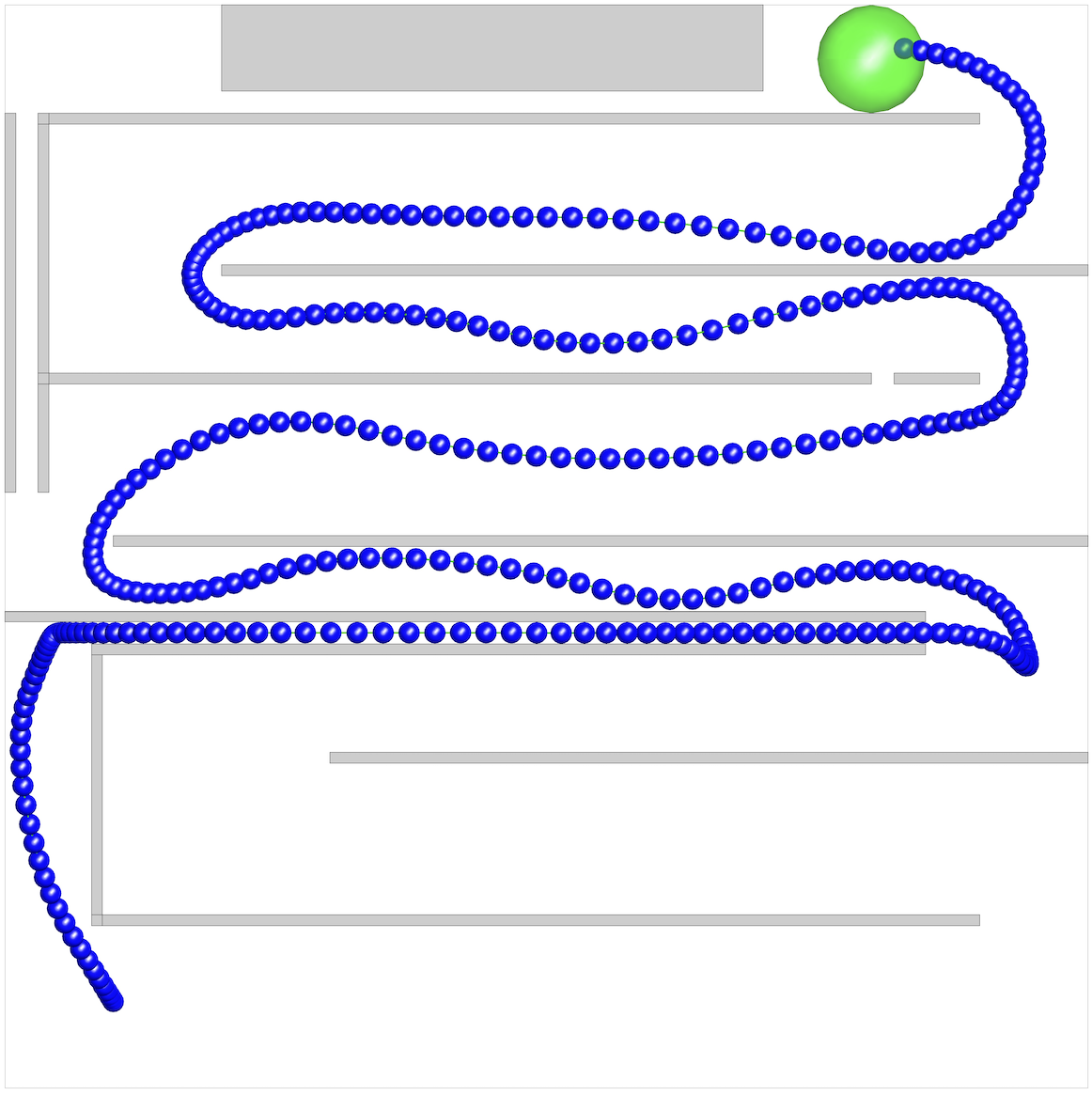}
        \caption{Intermediate Sol.}
        \label{fig:AOKPAX_intermediate_sol2}
    \end{subfigure}
    \begin{subfigure}[b]{0.32\columnwidth}
        \centering
        \includegraphics[width=\textwidth]{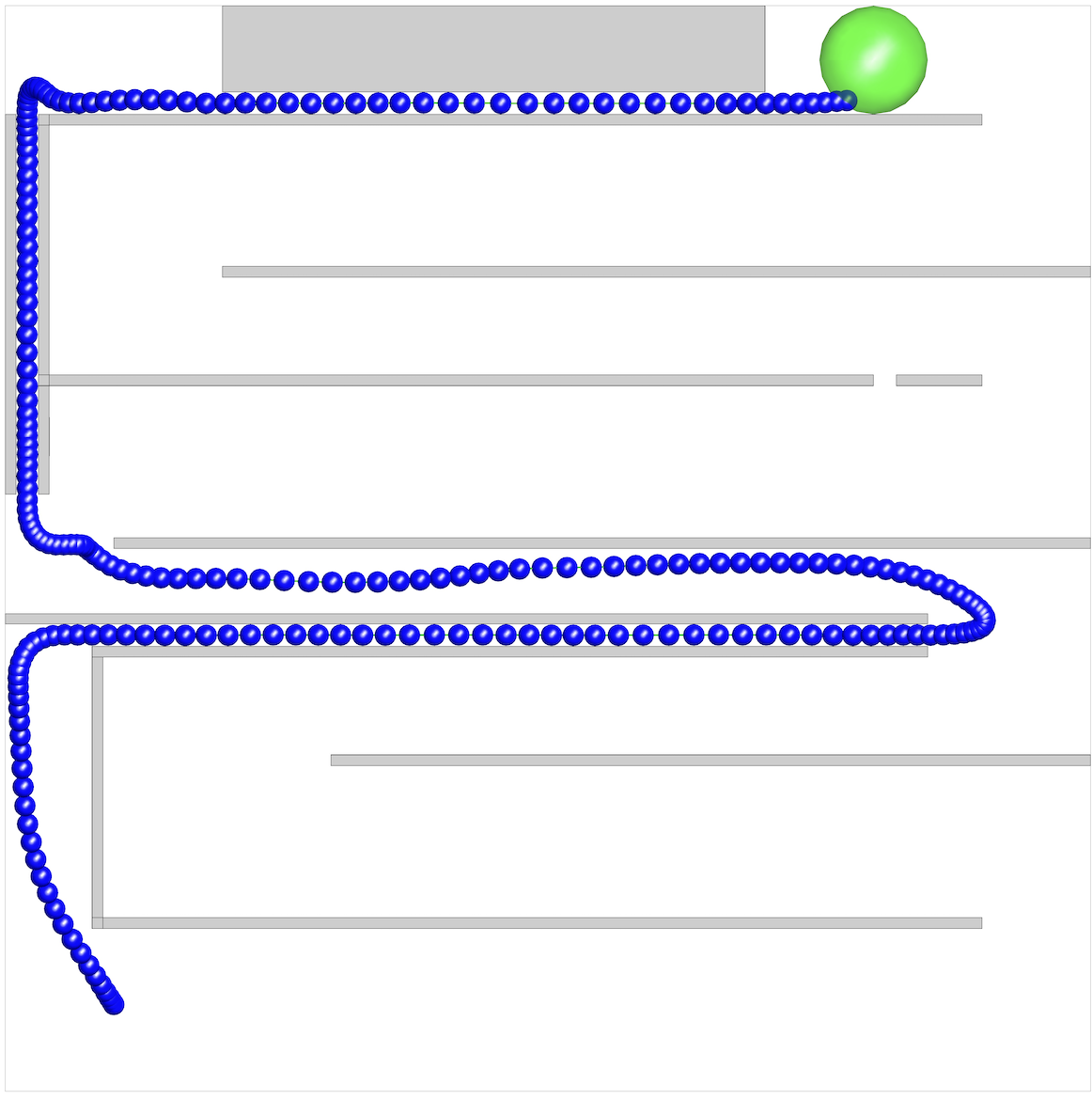}
        \caption{Final Solution}
        \label{fig:AOKPAX_final_sol2}
    \end{subfigure}
    \caption{Three solutions found by \aoalg using 6D Double Integrator dynamics and minimizing path length cost. 
    }
    \label{fig:AOKPAX_solutions2}
\end{figure}

On the other hand, parallel methodologies have been actively explored to overcome these serial limitations. CPU-based parallel approaches \cite{otte2013c, plaku2005sampling, sun2015high, ichnowski2012parallel, ichnowski2020concurrent, thomason2024motions} distribute the search workload across multiple cores to accelerate convergence. However, these methods are fundamentally constrained by the limited number of available CPU cores (typically tens) and the significant synchronization overhead required to maintain tree consistency, often yielding sub-linear speedups that are insufficient for high-rate replanning. Conversely, GPU-based algorithms such as \alg \cite{Perrault2025kinopax} utilize massive concurrency (thousands of threads) to quickly find solutions. While this yields millisecond-level planning times, it only finds feasible solutions without optimizing for a cost function. This reveals a fundamental gap: every existing parallel SBMP trades optimality guarantees for parallelism, while asymptotically near-optimal kinodynamic SBMP remains inherently serial. No prior work has demonstrated a massively parallel SBMP with asymptotic near-optimality guarantees for kinodynamic systems.


In this paper, we present \aoalg, a massively parallel
near-optimal kinodynamic SBMP designed for efficient
execution on modern many-core (GPU) processors. \aoalg
delivers \emph{low-cost} first solutions and guarantees
asymptotic \emph{near-optimality} over extended planning
durations (see Figure~\ref{fig:AOKPAX_solutions2}). The key
idea is to decompose the state space into fixed regions, each
tracking its own lowest-cost node via atomic updates, so that
expansion is concentrated on locally optimal trajectories.
This decomposition admits parallel execution: regions exist
independently of the search and can be updated by any thread
without serialization. Dominated nodes are pruned
concurrently, while dormant region-best nodes are periodically
reactivated to preserve completeness. We present a formal
analysis of \aoalg, proving its asymptotic near-optimality and
$\delta$-robust completeness. This is, to our knowledge, the
first such guarantee for a massively parallel kinodynamic
SBMP. Finally, we demonstrate through several benchmarks
that \aoalg is effective across a variety of motion planning
problems.

In short, our main contributions are: (i) the \emph{first} SBMP, to our knowledge, that simultaneously achieves massive (GPU-scale) parallelism and asymptotic near-optimality for kinodynamic systems, (ii) a novel analysis technique for parallel SBMPs, yielding the first proofs of probabilistic $\delta$-robust completeness and asymptotic $\delta$-robust near-optimality for a massively parallel kinodynamic planner, and (iii) several benchmarks and comparisons showing 
the efficiency and efficacy of \aoalg. 
 
Our results show that \aoalg achieves up to three orders-of-magnitude speedups over serial methods while attaining lower solution cost than both serial methods and \alg at comparable planning times. In 3D complex environments, \aoalg finds solutions in approximately $10$ ms for 6D systems and hundreds of milliseconds for 12D nonlinear systems, while continuously converging toward near-optimality with additional planning time.

%% file: sections/related.tex
\paragraph*{Geometric Motion Planning} Sampling-based motion planners (SBMPs) such as PRM~\cite{kavraki1996probabilistic} and RRT~\cite{lavalle2001rapidly} are foundational for geometric planning but are inherently serial. To improve planning rates, recent works have explored parallelization on both multi-core CPUs~\cite{otte2013c, plaku2005sampling, sun2015high, ichnowski2012parallel,ichnowski2020concurrent, thomason2024motions} and many-core GPUs~\cite{ichter2017group}. While CPU-based methods achieve modest speedups via fine-grained lock-free data structures, they are constrained by low core counts and high inter-thread communication costs. Conversely, GPU-based methods like \cite{ichter2017group} achieve orders-of-magnitude speedups by adapting CPU-based algorithms such as FMT*~\cite{janson2015fast}. However, these techniques rely on explicit boundary value problem solvers (steering functions) to connect states, restricting them strictly to geometric or simple linear problems and rendering them inapplicable to complex kinodynamic systems.

\paragraph*{Optimization-based Trajectory Optimization} A complementary line of work addresses kinodynamic planning through local trajectory optimization. Gradient-based methods ~\cite{schulman2014motion} optimize an initial trajectory
via sequential convex programming or covariant gradient
descent; sampling-based MPC methods~\cite{williams2017model} draw forward rollouts and
update controls via information-theoretic importance-weighted averaging. These
methods can produce smooth, locally optimal trajectories
rapidly, but they typically require an initial trajectory,
assume a fixed finite horizon, and lack global completeness
or optimality guarantees. In contrast, \aoalg performs
unbounded-horizon global planning from scratch for nonlinear
dynamics, providing asymptotic near-optimality by
construction.

\paragraph*{Kinodynamic Motion Planning \& Optimality}

For systems with complex differential constraints, kinodynamic planners such as RRT~\cite{lavalle2001randomized} and EST~\cite{hsu1997path} extend trees via forward propagation. To achieve \emph{asymptotic near-optimality}, algorithms like SST~\cite{li2016asymptotically} and AO-RRT~\cite{Hauser2016aox} introduce cost-based pruning and selection criteria. SST, in particular, maintains a sparse tree by keeping only the lowest-cost node within a local neighborhood and pruning dominated alternatives. \aoalg shares this principle of retaining locally best nodes, but differs in two fundamental ways. SST maintains a single global witness structure that is updated sequentially, making it inherently serial, whereas \aoalg partitions the state space into fixed regions whose costs are updated in parallel. Existing parallel kinodynamic approaches typically employ \emph{coarse-grained} parallelism (running multiple independent trees)~\cite{746692}, which improves average-case time-to-solution but does not accelerate the convergence rate of a single optimal plan.

Most relevant to our work is \alg~\cite{Perrault2025kinopax}, which demonstrates that massive GPU parallelism can solve the \emph{feasibility} problem for kinodynamic systems in milliseconds. \alg decomposes the standard SBMP loop into three parallel subroutines: it samples controls for every active node simultaneously, propagates them in parallel, and adds valid children to the tree in parallel. A fixed array-based tree representation and a spatial decomposition of the state space enable threads to operate independently, while an expected-tree-size hyperparameter bounds memory usage on the GPU. However, \alg treats all nodes equally during expansion and performs no cost comparison or pruning, so it finds feasible trajectories quickly but cannot improve their quality over time.

\aoalg retains \alg's subroutines (array-based tree, atomic propagation, expected-tree-size budgeting) but introduces a novel search strategy to enable cost refinement and asymptotic optimality. Specifically, \aoalg introduces three new mechanisms: (i) each region of the spatial decomposition now tracks its lowest-cost node, so that only locally best nodes are selected for further expansion, (ii) a node management scheme deactivates cost-dominated nodes and periodically reactivates them to preserve completeness, and (iii) a concurrent pruning phase removes suboptimal nodes without serializing the propagation phase. While \alg explores uniformly to find any feasible solution, \aoalg progressively concentrates computation on promising trajectories, and provides asymptotic near-optimality guarantees.

%% file: sections/problem.tex
Consider a robotic system in a bounded workspace $W \subset \mathbb{R}^d$, where $d \in \{2, 3\}$, with a finite set of obstacles $\mathcal{O}$.
The robot's motion is constrained to the dynamics
\begin{equation}
\label{eq:diffEq}
\dot{x}(t) = f(x(t), u(t)), 
\end{equation}
where 
$x(t) \in X \subset \mathbb{R}^n$ is the state, and $u(t) \in U \subset \mathbb{R}^N$ is the control at time $t$. The state space $X$ and control space $U$ are assumed to be compact. Function $f:X \times U \to \mathbb{R}^n$ is the vector field and assumed to be Lipschitz continuous with respect to both arguments, i.e, there exist constants \(K_x, K_u > 0\) such that for all \(x, x' \in X\) and \(u, u' \in U\),
$$\|f(x, u) - f(x', u')\| \leq   K_x \|x - x'\| + K_u \|u - u'\|.$$
In addition, we assume that the system dynamics satisfy Chow’s condition~\cite{Chow1940/1941}, which implies that the system is small-time locally accessible.

Beyond the differential constraints in~\eqref{eq:diffEq} and obstacles in $\mathcal{O}$, the robot must satisfy state constraints (e.g., velocity limits).
We define the \textit{valid} state space $\rv{X_{\free}} \subseteq X$ as the set of states that are collision-free and satisfy all state constraints. 

Given an initial state $x_{\init} \in \rv{X_{\free}}$, a time horizon \(t_{f} \geq 0\), and a control trajectory $\mathbf{u}: [0,t_{f}] \to U$, 
a unique \textit{state trajectory} $\traj: [0,t_{f}] \to X$ is obtained such that
\begin{align}
    \traj(t) = x_{\init} + \int_{0}^{t} f(\traj(\tau),\mathbf{u}(\tau)) d \tau \qquad \forall t \in [0,t_f].
    \label{eq:integral}
\end{align}
Trajectory $\traj$ is \emph{valid} if, $ \forall t \in [0,t_f]$, $\traj(t) \in \rv{X_{\free}}$. The objective of kinodynamic motion planning is to find a valid trajectory $\traj$ that reaches a given goal set $X_\goal \subseteq \rv{X_{\free}}$. 


In addition to feasibility, we seek to optimize the cost of the trajectory.
Let $\Traj$ denote the set of trajectories with finite durations.  The cost of a trajectory is defined by a function $\cost: \Traj \to \mathbb{R}_{\geq 0}$ that assigns to each trajectory $\traj \in \Traj$ a non-negative value $\cost(\traj) \in \mathbb{R}_{\geq 0}$.
We assume that $\cost$ is a smooth, continuous function satisfying the additivity, monotonicity, and non-degeneracy properties \cite{li2016asymptotically}, as formally stated below.

\begin{assumption}[\!{\cite{li2016asymptotically}}]
    \label{assumption: cost}
    The cost function $\cost$ is Lipschitz continuous, i.e, there exists constant $K_c > 0$ such that
\[
|\cost(\traj) - \cost(\traj')| \leq K_c \ \sup_{t \in [0,t_f]} \|\traj(t) - \traj'(t)\|
\]
for all $\traj,\traj'\in \Traj$ 
with the same start state, i.e., $\traj(0) = \traj'(0)$. Furthermore, consider a trajectory $\traj \in \Traj$ with duration $t_f > 0$, and 
for $t \in [0, t_f]$,
define $\traj^{t}$ and $\traj_{t}$ as its \emph{prefix} up to time $t$ and \emph{suffix} from time $t$, respectively, so that their concatenation $\traj^{t} \bullet \:\: \traj_{t} = \traj$.  
Then, it holds that, for all $t \in [0,t_f]$:
\begin{align*}
    &\text{Additivity:} && \cost(\traj) = \cost(\traj^t) + \cost(\traj_t) \\
    &\text{Monotonicity:} && \cost(\traj^t) \leq \cost(\traj)\\
    &\text{Non-degeneracy:} && \exists K_c > 0 \text{ s.t. } \\
    & && \qquad t_f - t \leq K_c \, |\cost(\traj) - \cost(\traj^t)|.
\end{align*}
\end{assumption}

Ideally, one seeks a valid trajectory $\traj$ that reaches the goal set $X_{\goal}$ while minimizing $\cost(\traj)$. However, computing such an optimal trajectory is notoriously difficult. Instead, we aim to compute a near-optimal solution \emph{very quickly}.

\begin{definition}[Near-optimal trajectory]
    \label{def:near optimal}
    Let $\Traj_\text{sol} \subseteq \Traj$ be the set of solution trajectories to a given motion planning problem, i.e., 
    $\Traj_\text{sol} = \{\traj \in \Traj \mid \traj(t) \in X_\free \;\;\; \forall t \in [0,t_f], \;\;  \traj(t_f) \in X_\goal
    \},$
    and
    denote by $c^*$ the minimum cost over all solution trajectories, i.e.,
    $c^* = \min_{\traj \in \Traj_\text{sol}} \cost(\traj).$
    A trajectory $\traj^+ \in \Traj_\text{sol}$ is called a \emph{near-optimal} solution if, for a $\beta > 0$,
    \begin{align*}
         \cost(\traj^+) \leq (1+\beta) c^*.
    \end{align*}
\end{definition}

\begin{problem}[Near-Optimal Kinodynamic Motion Planning]
    \label{problem}  
    Consider a robot with dynamics in \eqref{eq:diffEq} operating in a workspace $W$ with obstacle set $\mathcal{O}$. Given an initial state $x_{\init} \in X_{\free} \subseteq X$, a goal region $X_{\goal} \subseteq X_{\free}$, and a trajectory cost function $\cost$ satisfying Assumption~\ref{assumption: cost},  \emph{efficiently} compute a control trajectory $\mathbf{u} : [0,t_f] \to U$ whose induced state trajectory is a near-optimal solution as defined in Definition~\ref{def:near optimal}.
\end{problem}

While Problem~\ref{problem} is a standard near-optimal kinodynamic motion planning problem addressed by algorithms such as SST~\cite{li2016asymptotically}, it remains computationally challenging for at the millisecond-to-second timescales required for responsive deployment.
Recent GPU-based parallel solvers \cite{Perrault2025kinopax} have enabled
finding feasible kinodynamic trajectories tractable; however, the added complexity of optimizing for a cost function and ensuring near-optimality changes the algorithmic requirements. In this work, we focus on designing a highly efficient, GPU-based algorithm that provides asymptotic near-optimality guarantees.

%% file: sections/Algorithm.tex
To solve Problem~\ref{problem}, we propose \emph{Near Optimal Kinodynamic Parallel Accelerated eXpansion} (\aoalg), a highly parallel kinodynamic SBMP designed to utilize the throughput of many-core (GPU) architectures while providing asymptotic near-optimality properties. \aoalg builds a sparse trajectory tree by decomposing the SBMP iterative operations (node selection, node extension, and node pruning) into three massively parallelized subroutines. Each subroutine is optimized for efficient execution on highly parallel processors and reduce communication overhead such as CPU-GPU interactions.

During each iteration of \aoalg, multiple nodes from the existing tree are extended concurrently via random control sampling. This extension is governed by a spatial decomposition $\mathcal{R}$ to guide the search process and systematically optimize for low cost solutions. Additionally, this decomposition promotes thread independence during node addition and selection phases. After extensions are completed, non-promising nodes are pruned in parallel, and a new set of promising nodes with low cost is selected concurrently for expansion in the subsequent iteration. By continuously tracking the lowest-cost node per region, \aoalg progressively refines cost estimates, selecting only promising nodes for further propagation, and converges towards near-optimal solutions.

\aoalg inherits from \alg~\cite{Perrault2025kinopax} the decomposition-based search representation and parallel tree propagation that together enable thread-independent execution. The novel components introduced in \aoalg are: (i) a cost-aware spatial decomposition that guides expansion toward
locally optimal nodes, replacing \alg's uniform exploration, (ii) cost-aware node management that focuses computation
on promising regions of the tree while preserving
completeness, and (iii) a concurrent dominance-based pruning mechanism that improves solution cost without serializing the parallel propagation phase. Together, these enable \aoalg to achieve
asymptotic near-optimality and probabilistic
$\delta$-robust completeness.

Below, we present the core algorithm and its subroutines, and, in Appendix~\ref{appendix: hyperparameters}, we provide low-level hyper-parameter tuning discussion and implementation details.

\subsection{Core Algorithm}

\begin{figure*}[ht!]
\label{AOKPAX_iteration}
    \centering
    \begin{subfigure}[t]{0.04\textwidth}
        \centering
        \vspace{-39mm}
        \includegraphics[width=\textwidth]{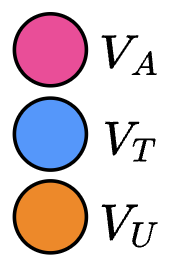}
    \end{subfigure}
    \hspace{-3mm}
    ~
    \begin{subfigure}[b]{0.18\textwidth}        \centering
        \includegraphics[width=\textwidth]{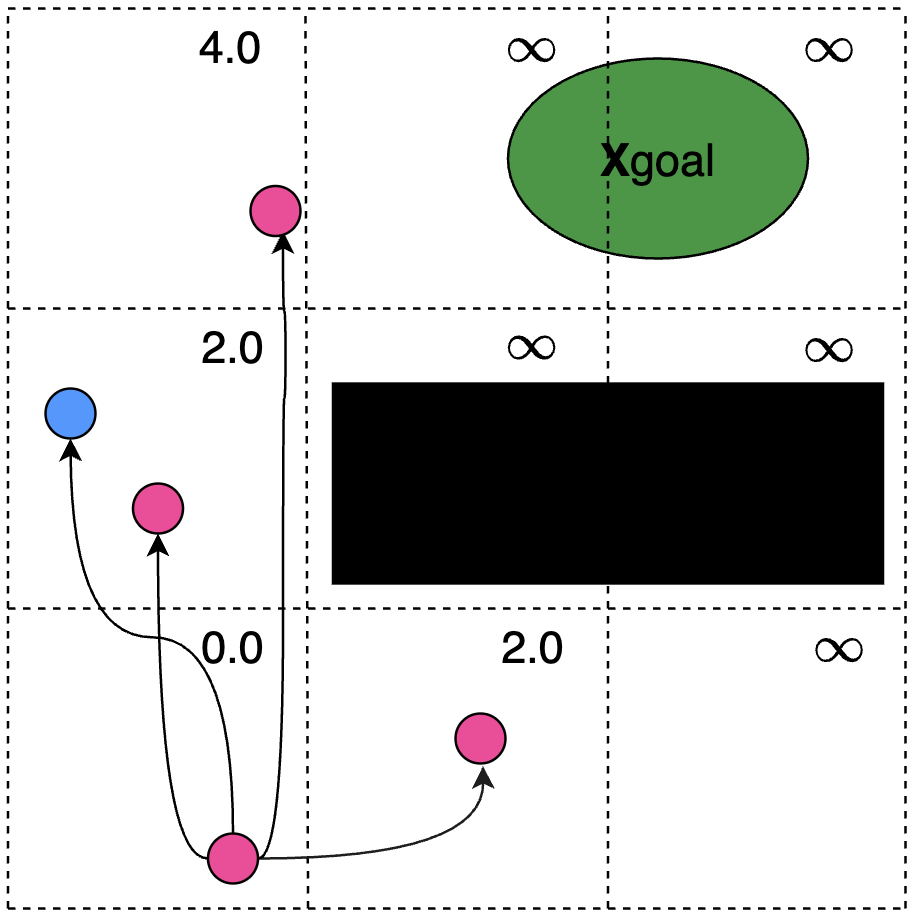}
         \caption{}
        \label{fig:AOKOAX_iteration1}
    \end{subfigure}
    \hspace{-3mm}
    ~
    \begin{subfigure}[b]{0.18\textwidth}
        \centering
        \includegraphics[width=\textwidth]{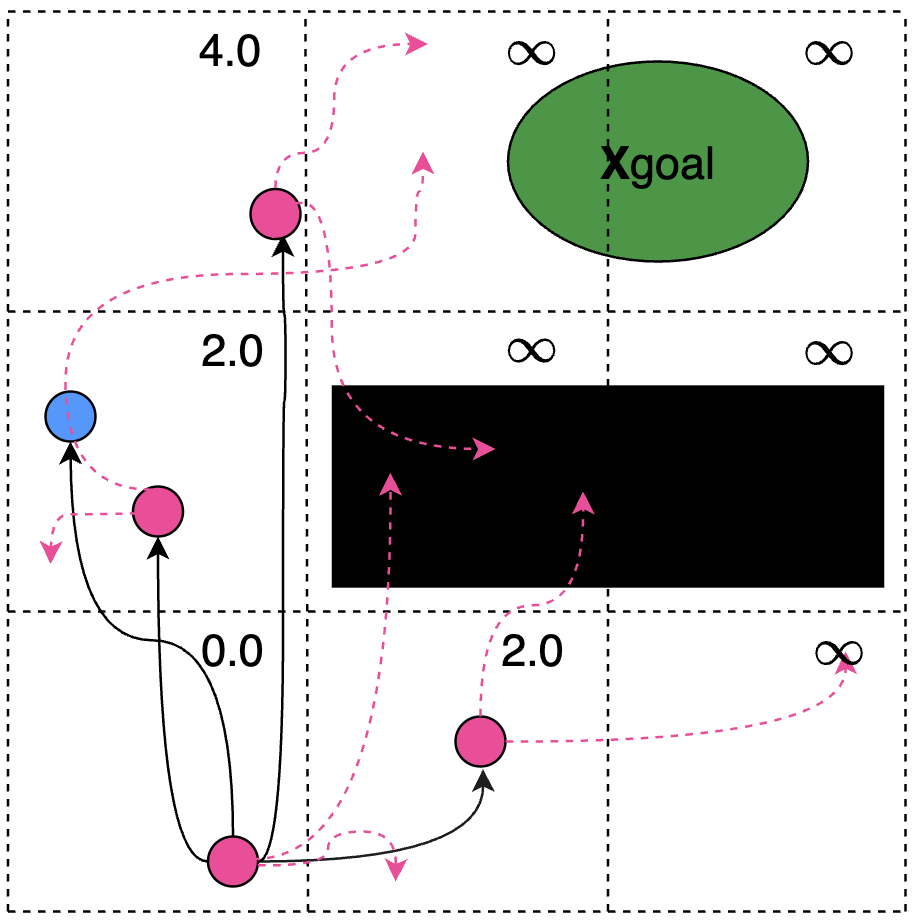}
        \caption{}
        \label{fig:AOKOAX_iteration2}
    \end{subfigure}
    \hspace{-3mm}
    ~
    \begin{subfigure}[b]{0.18\textwidth}
        \centering
        \includegraphics[width=\textwidth]{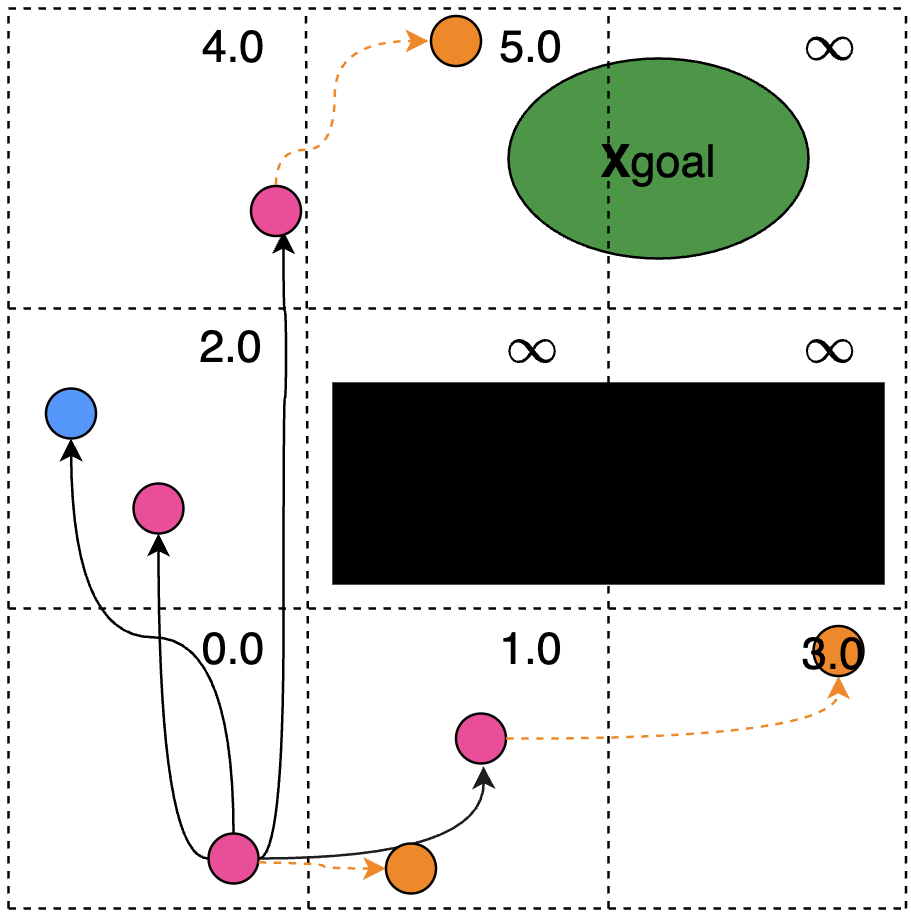}
        \caption{}
        \label{fig:AOKOAX_iteration3}
    \end{subfigure}
    \hspace{-3mm}
    ~
    \begin{subfigure}[b]{0.18\textwidth}
        \centering
        \includegraphics[width=\textwidth]{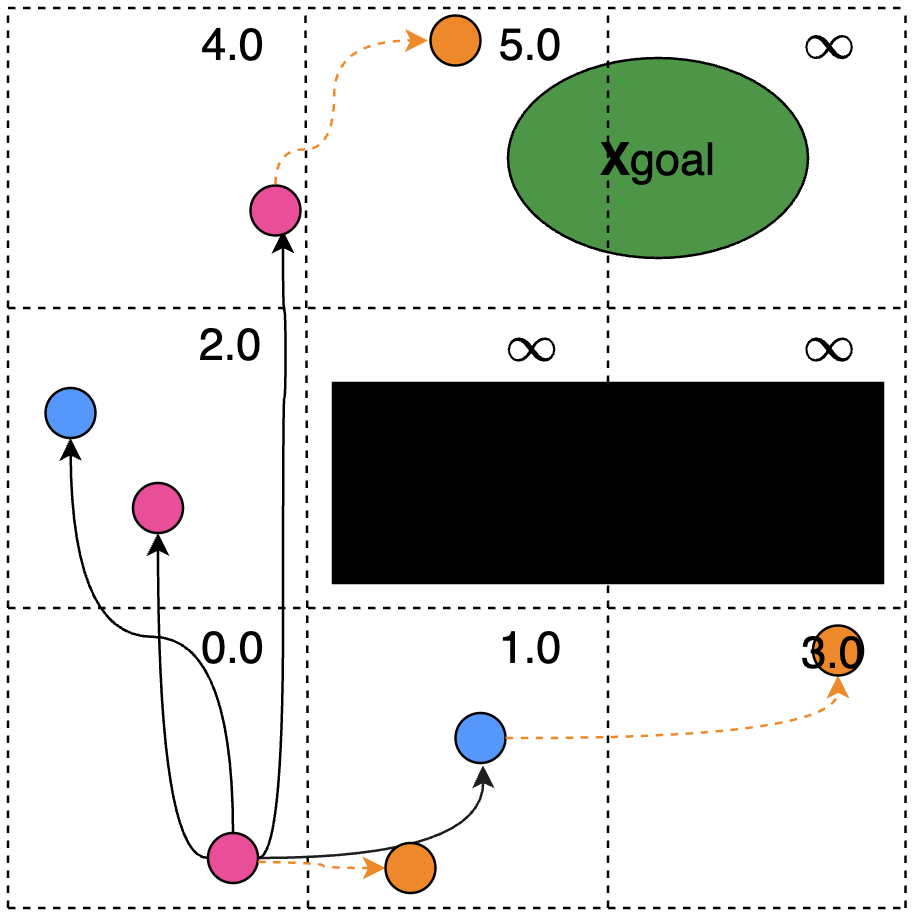}
        \caption{}
        \label{fig:AOKOAX_iteration4}
    \end{subfigure}
    \hspace{-3mm}
    ~
    \begin{subfigure}[b]{0.18\textwidth}
        \centering
        \includegraphics[width=\textwidth]{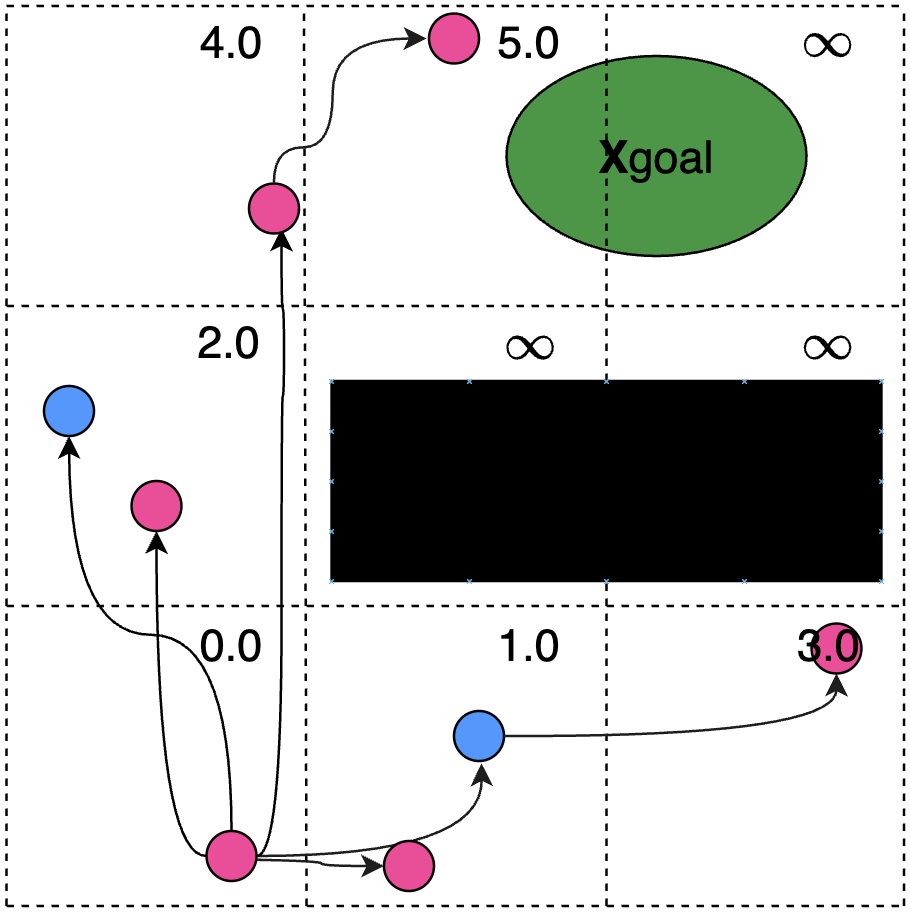}
        \caption{}
        \label{fig:AOKOAX_iteration5}
    \end{subfigure}
    \hspace{-3mm}
    \caption{ \rv{ Illustration of the \aoalg expansion process:  
(a) Current node sets \( V_A \) and \( V_{T} \). Numbers within grid cells indicate the lowest-cost trajectories reaching each region \( \mathcal{R}_i \).
(b) Parallel expansion of nodes in \( V_A \) with branching factor \( \lambda = 2 \).
(c) Addition of selected nodes to the unexplored set \( V_U \) and corresponding updates to region costs.
(d) Pruning of nodes from \( V_A \) based on newly acquired trajectory cost information.
(e) Updated active set \( V_A \), ready for the next iteration of node expansion.
}}
    \label{fig:AOKPAX_overview}
\end{figure*}

Here, we provide a detailed description of \aoalg with pseudocode in Algorithms~\ref{algo:ao_alg}-\ref{algo:ao_TreeUpdate}
as well as an illustration of one cycle of \aoalg in Figure~\ref{fig:AOKPAX_overview}. \aoalg organizes tree nodes into four distinct sets: $V_A$, $V_I$, $V_T$, and $V_U$. Specifically, $V_A$ contains nodes selected for parallel expansion (Algorithm~\ref{algo:ao_Propagate}); $V_I$ includes temporarily inactive nodes that may be reactivated for expansion in future iterations (Algorithm~\ref{algo:ao_NodePruning}); $V_T$ comprises terminal nodes permanently pruned from the tree, and thus not considered in subsequent iterations (Algorithm~\ref{algo:ao_NodePruning}); and finally, $V_U$ consists of newly generated promising nodes produced during node propagation (Algorithm~\ref{algo:ao_Propagate}). Additionally, \aoalg partitions the state space into distinct hyper-cube regions with diagonal length $\delta$, denoted by $\mathcal{R}$, and maintains the lowest-cost node within each region to guide efficient exploration.

\subsubsection{Initialization}

\label{sec:OKPAX_initialization}

In Algorithm~\ref{algo:ao_alg}, \texttt{\aoalg} takes as input an initial state $x_{\text{init}} \in X_\free$, a goal region $X_{\text{goal}} \subseteq X_\free$, a maximum execution time $t_{\text{max}}$, a branching factor $\lambda$, and a maximum inactive node count $I_{\text{max}}$. In Lines 1-2, the initial state $x_{\text{init}}$ is set as the root node of the tree $\mathcal{T}$ and placed into the active set $V_A$. The inactive set $V_I$, terminal set $V_T$, and unexplored set $V_U$ are initialized as empty. Lines 3-5 initialize each region in the spatial decomposition $\mathcal{R}$ with infinite cost, while the solution trajectory $\traj$ is initialized as \texttt{null} with infinite cost.

\begin{algorithm}[t]
    \caption{\texttt{\aoalg}}
    \label{algo:ao_alg}
    \SetKwInOut{Input}{Input}\SetKwInOut{Output}{Output}
    \Input{Initial state $x_{\init}$, goal region $X_{\goal}$, planning time $t_{max}$, parallel expansion $\lambda$, decomposition diagonal $\delta$, inactivity threshold $I_{max}$}
    \Output{Near-optimal solution trajectory $\traj$}
    
    \SetKwFunction{Propagate}{Propagate}
    \SetKwFunction{PruneNodes}{PruneNodes}
    \SetKwFunction{UpdateTree}{UpdateTree}

    $\mathcal{T} \gets$ Initialize tree with root node $x_{\init}$\\
    $V_A \gets \{x_{\init}\}$, $V_{I}, V_{U}, V_T \gets \emptyset$ \\
    Construct decomposition $\mathcal{R}=\{ \mathcal{R}_i \}_{i=1}^{N}$ with diameter $\delta$\\
    $\mathcal{R}=\{\mathcal{R}_i\}_i^{N}$ with $cost(\mathcal{R}_i) = \infty$\\
    Initialize $\traj = null$ with $\cost(\traj) = \infty$\\
    \While{$ElapsedTime < t_{max}$}{
        \Propagate{$V_A, V_U, \mathcal{R}, \lambda$} \\
        \PruneNodes{$\mathcal{T}, V_A, V_I, V_T, \mathcal{R}, I_{count}, I_{max}$} \\
        \UpdateTree{$\mathcal{T}, V_A, V_U, \mathcal{R}$} \\
    }
    \KwRet{$\traj$}
\end{algorithm}

\subsubsection{Node Extension}

After initialization, the main loop of \aoalg begins (Algorithm~\ref{algo:ao_alg}, Lines~5-8). In each iteration, the \texttt{Propagate} subroutine (Algorithm~\ref{algo:ao_Propagate}, Figures~\ref{fig:AOKOAX_iteration1} and \ref{fig:AOKOAX_iteration2}) propagates nodes in the active set $V_A$ using massive parallelism. The inputs to \texttt{Propagate} include $V_A$, $V_U$, $\mathcal{R}$, and parameter $\lambda \in \mathbb{N}^+$. Specifically, each node $x \in V_A$ undergoes $\lambda$ parallel expansions, resulting in a total of $\lambda \cdot |V_A|$ parallel threads, where $|V_A|$ denotes the cardinality of $V_A$. For each thread, a random control $u \in U$ and a duration $dt \in (0, T_{\text{prop}}]$ are sampled, with $T_{\text{prop}} > 0$ being a user-defined maximum propagation time. The node's state $x$ is then propagated using the system dynamics in \eqref{eq:diffEq}, generating a new candidate state $x'$ (Algorithm~\ref{algo:ao_Propagate}, Lines~3-4, Figure~\ref{fig:AOKOAX_iteration2}).

Subsequently, the validity of the trajectory segment between states $x$ and $x'$, denoted $\overline{xx'}$, is evaluated against $X_{\text{valid}}$. If the trajectory segment is valid, the spatial region $\mathcal{R}_i$ corresponding to the new state $x'$ is identified (Algorithm~\ref{algo:ao_Propagate}, Line~6). The incremental cost from $x$ to $x'$, i.e., $\cost(\overline{xx'})$, is computed, and the cumulative cost from the root $x_{\init}$ to $x'$ is updated accordingly (Algorithm~\ref{algo:ao_Propagate}, Line~7).

If the new state $x'$ improves upon the previously recorded cost for region $\mathcal{R}_i$, the cost associated with that region is updated (Algorithm~\ref{algo:ao_Propagate}, Line~8). Furthermore, if $x'$ remains the lowest recorded cost for region $\mathcal{R}_i$, it is added to the unexplored node set $V_U$ (Algorithm~\ref{algo:ao_Propagate}, Lines~9-10, Figure~\ref{fig:AOKOAX_iteration3}). This selective admission criterion ensures that nodes added to $V_U$ consistently improve upon the best trajectories found, guiding the search toward near-optimal solutions while maintaining memory efficiency.


\begin{remark}
Since the \texttt{Propagate} subroutine executes as a parallel process, a thread may insert a node into the unexplored set $V_U$ even if another thread identifies a lower-cost trajectory to the same region. The higher-cost nodes introduced initially are removed during the \texttt{UpdateTree} subroutine, ensuring only the lowest-cost nodes remain in the tree $\mathcal{T}$.
\end{remark}

\begin{algorithm}[t]
    \caption{\texttt{Propagate}}
    \label{algo:ao_Propagate}
    \SetKwInOut{Input}{Input}\SetKwInOut{Output}{Output}
    \Input{$V_A, V_U, \mathcal{R}, \lambda$}
    \Output{Updated $V_U$}
    \ForEach{$x \in V_A$ }{
        \For{$i = 1, \dots, \lambda$}{
            \rv{Randomly}
            sample $u$ and $dt$ \\
            $x' \gets \rv{\texttt{PropagateODE}}(x, u, dt)$ \\ 
            \If{the trajectory from $x$ to $x'$ is valid}{
                Map $x'$ to region $\mathcal{R}_i$ \\
                $\cost(\overline{x_{\init}x'}) \gets \cost(\overline{x_{\init}x}) + \cost(\overline{xx'})$ \\
                $cost(R_i) \gets \min \{cost(R_i), \cost(\overline{x_{\init}x'}) \}$ \\
               \If{$\cost(\overline{x_{\init}x'}) = cost(R_i)$}{
                    Add $x'$ to $V_{U}$ 
                }
            }
        }
    }
\end{algorithm}

\subsubsection{Node Selection}

After propagating nodes in $V_A$, the algorithm executes the \texttt{PruneNodes} subroutine (Algorithm~\ref{algo:ao_alg}, Line~7; Algorithm~\ref{algo:ao_NodePruning}), which classifies each node \( x \in \mathcal{T} \) as active, inactive, or terminal based on newly acquired search information. This subroutine operates as a massively parallel procedure, assigning one thread per node.

\begin{algorithm}[t]
    \caption{\texttt{PruneNodes}}
    \label{algo:ao_NodePruning}
    \SetKwInOut{Input}{Input}
    \SetKwInOut{Output}{Output}
    \Input{$\mathcal{T}, V_A, V_I, V_T, \mathcal{R}, I_{count}, I_{max}$}
    \Output{Updated $V_A, V_I, V_T, I_{count}$}
       
        \ForEach{$x \in \mathcal{T}$}{
            Map $x$ to region $\mathcal{R}_i$ \\
            \If{$\cost(\overline{x_{\init}x}) = cost(R_i)$ and $x \in V_I$}{
                Increment $I_{count}(x)$\\
                \If{$I_{count}(x) > I_{max}$}{
                    Add $x$ to $V_A$\\  
                    \KwRet{}\\
                }
            }
            \If{$I_{count}(x) > I_{max}$ and $\cost(\overline{x_{\init}x}) = cost(R_i)$}
            {
                \KwRet{}\\
            }
            \If{$\cost(\overline{x_{\init}x}) > cost(R_i)$}
            {
                Add $x$ to $V_T$\\  
                \KwRet{}\\
            }
            \If{$\exists$ ancestor $x_p$ of $x$, $\cost(\overline{x_{init}x_p}) > cost(R_p)$}
            {
                Add $x$ to $V_I$\\  
                \KwRet{}
            }
        }

\end{algorithm}

Each thread begins by mapping its node \( x \) to its corresponding region \( \mathcal{R}_i \) (Algorithm~\ref{algo:ao_NodePruning}, Line~2). If a node \( x \in V_I \) (currently inactive) no longer represents the lowest cost in its region, it is moved permanently to the terminal set \( V_T \) (Algorithm~\ref{algo:ao_NodePruning}, Lines~10-12). Conversely, if \( x \in V_I \) remains the lowest-cost node in its region, its inactivity counter \( I_{\text{count}}(x) \) is incremented (Algorithm~\ref{algo:ao_NodePruning}, Lines~3-4). If this counter exceeds a user-defined threshold \( I_{\text{max}} \), indicating prolonged inactivity without improvement in its region, node \( x \) is reactivated and returned to the set \( V_A \) for further expansion (Algorithm~\ref{algo:ao_NodePruning}, Lines~5-7). This mechanism ensures essential nodes re-enter active exploration, maintaining completeness.

Nodes currently in \( V_A \) are also evaluated for pruning. A node \( x \in V_A \) is pruned and moved to \( V_T \) if its trajectory cost exceeds the lowest recorded cost for its region (Algorithm~\ref{algo:ao_NodePruning}, Lines~10-12). Alternatively, if node \( x \) has the lowest cost within its region but any of its ancestor nodes exceeds the minimum recorded cost for their respective regions, node \( x \) is moved to the set \( V_I \) (Algorithm~\ref{algo:ao_NodePruning}, Lines~13-15, Figure~\ref{fig:AOKOAX_iteration4}). This selective pruning strategy enables \aoalg to concentrate computational resources on expanding a focused subset of promising nodes with a higher branching factor.

After pruning, the \texttt{UpdateTree} subroutine is invoked (Algorithm~\ref{algo:ao_alg}, Line~8; Algorithm~\ref{algo:ao_TreeUpdate}). This subroutine concurrently removes nodes in \( V_U \) that are not the minimum-cost nodes in their respective regions, adds the remaining nodes in \( V_U \) to \( \mathcal{T} \), and checks whether a better solution has been found. \texttt{UpdateTree} receives as input \( \mathcal{T} \), \( V_A \), \( V_U \), and \( \mathcal{R} \), with each thread responsible for processing a single node from \( V_U \). Each thread retrieves the corresponding region \( \mathcal{R}_i \) for its node \( x \) (Algorithm~\ref{algo:ao_TreeUpdate}, Line~2). If \( x \) represents the lowest-cost node to reach region \( \mathcal{R}_i \), it is added to both \( V_A \) and \( \mathcal{T} \) (Algorithm~\ref{algo:ao_TreeUpdate}, Lines~3-4, Figure~\ref{fig:AOKOAX_iteration5}). Subsequently, if node \( x \) satisfies the goal criteria and its cost improves upon the current best-known solution, the solution trajectory \( \traj \) and its associated cost are updated accordingly (Algorithm~\ref{algo:ao_TreeUpdate}, Lines~5-7).

\aoalg repeats the main loop of \texttt{Propagate}, \texttt{PruneNodes}, and \texttt{UpdateTree} until the user-defined time limit \( t_{\text{max}} \) is exceeded, at which point the best-found solution trajectory \( \traj \) is returned. Figures~\ref{fig:AOKPAX_solutions2} and \ref{fig:AOKPAX_solutions} illustrate how \aoalg progressively improves the quality of its solution over time.
\begin{figure}[h]
    \centering
     \begin{subfigure}[b]{0.32\columnwidth}
        \centering
        \includegraphics[width=\textwidth]{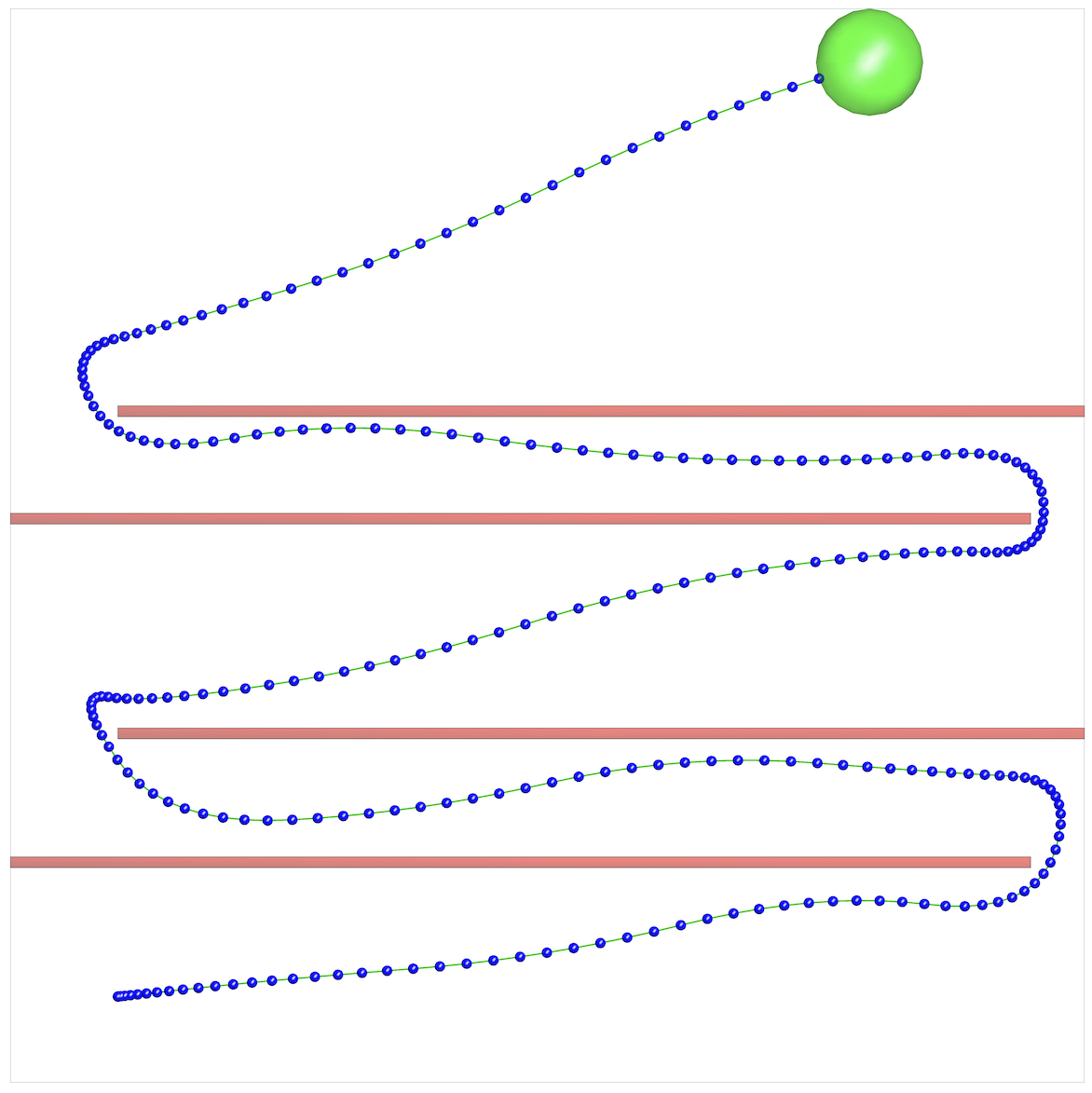}
        \caption{First Solution}
        \label{fig:AOKPAX_first_sol}
    \end{subfigure}
    \begin{subfigure}[b]{0.32\columnwidth}
        \centering
        \includegraphics[width=\textwidth]{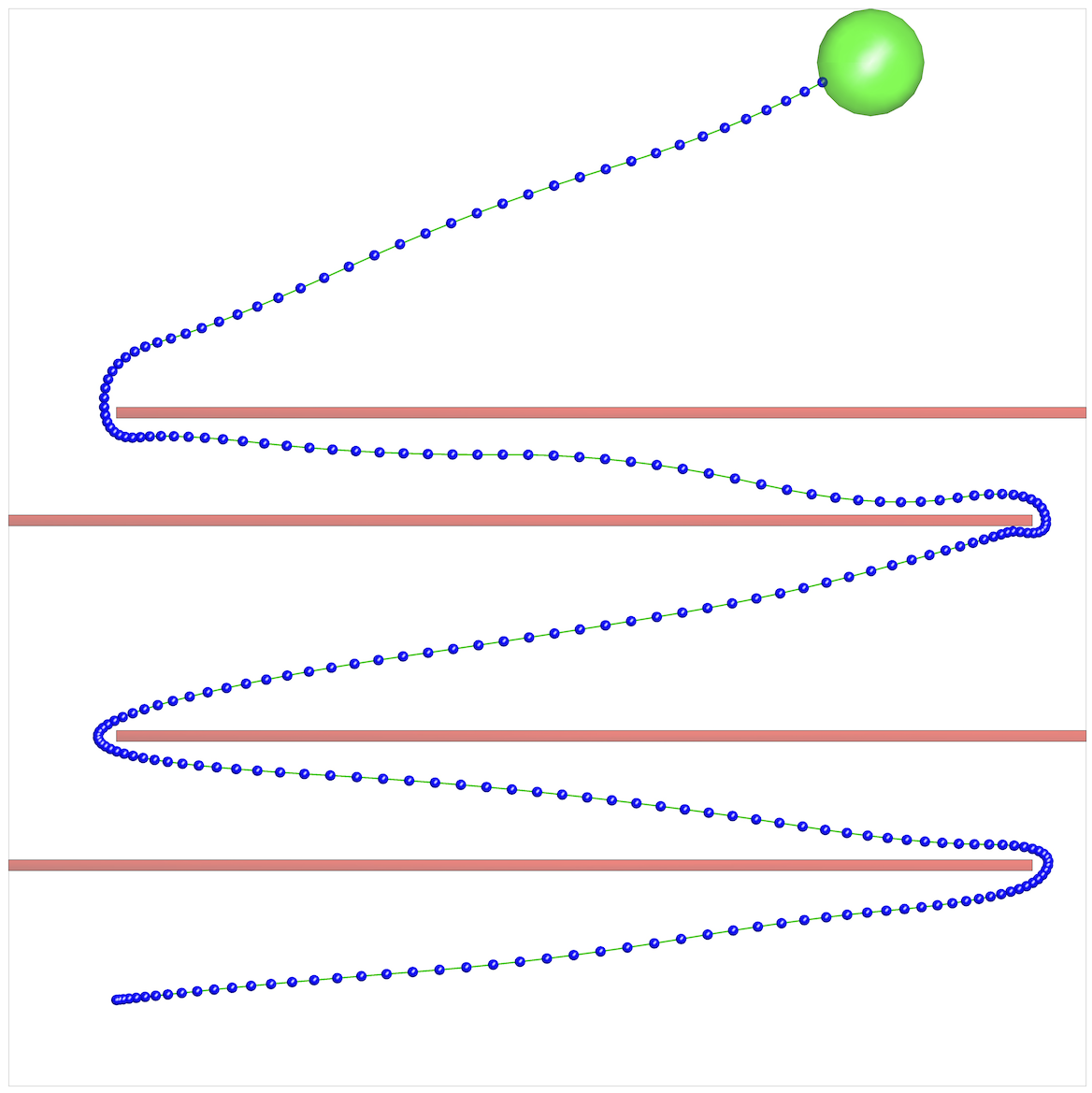}
        \caption{Intermediate}
        \label{fig:AOKPAX_intermediate_sol}
    \end{subfigure}
    \begin{subfigure}[b]{0.32\columnwidth}
        \centering
        \includegraphics[width=\textwidth]{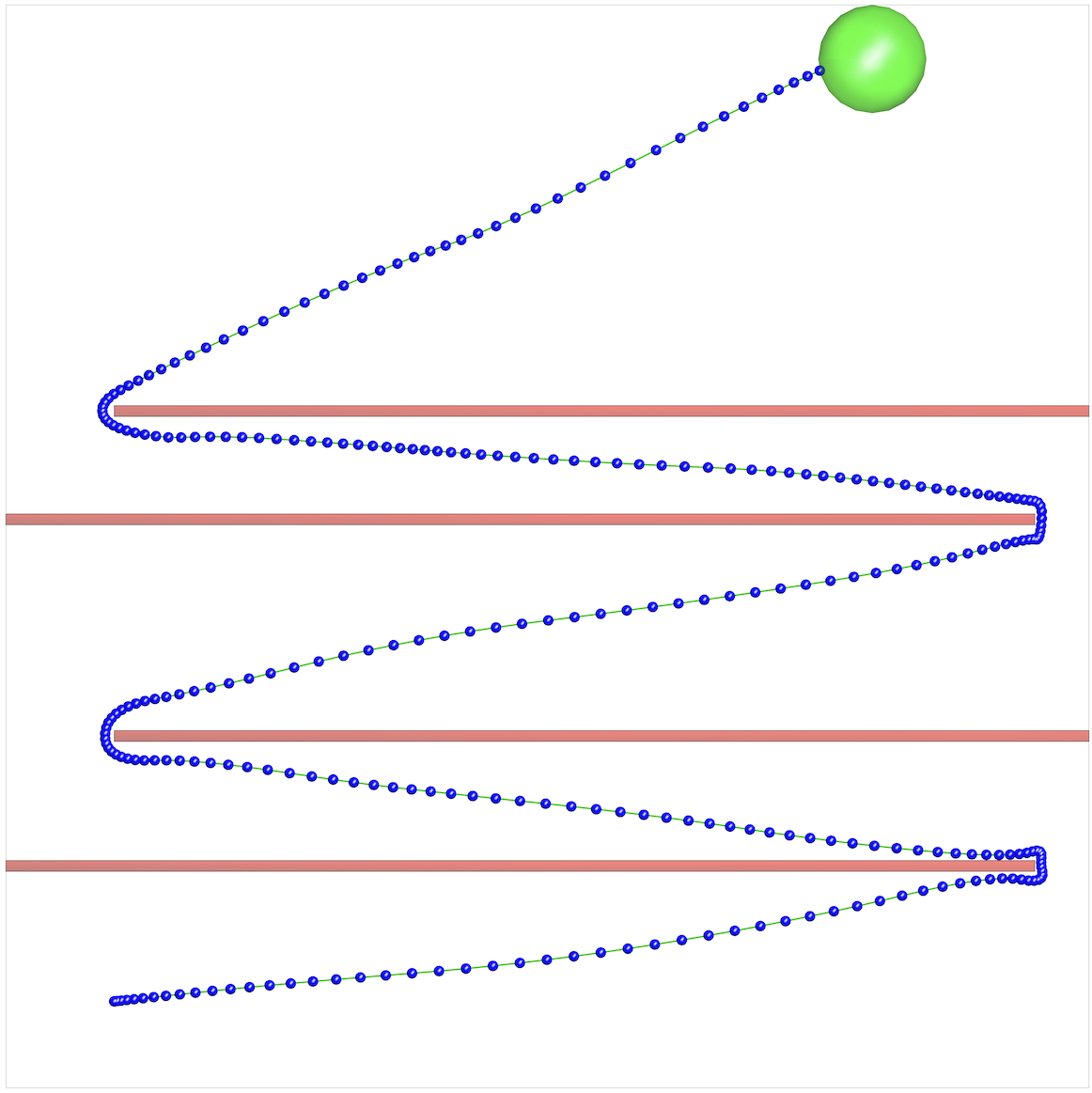}
        \caption{Final Solution}
        \label{fig:AOKPAX_final_sol}
    \end{subfigure}
    \caption{Three solutions found by \aoalg using 6D Double Integrator dynamics with improved cost over time.
    Subfigure (a) is the first solution, (b) is an intermediate solution, and (c) is the final solution.
    }
    \label{fig:AOKPAX_solutions}
\end{figure}



\begin{algorithm}[t]
    \caption{\texttt{UpdateTree}}
    \label{algo:ao_TreeUpdate}
    \SetKwInOut{Input}{Input}
    \SetKwInOut{Output}{Output}
    \Input{$\mathcal{T}, V_A, V_U, \mathcal{R}$}
    \Output{Updated $V_A, \mathcal{T}, \traj$}
         \ForEach{$x \in \mathcal{V_U}$}{
            Map $x$ to region $\mathcal{R}_i$ \\
            \If{$\cost(\overline{x_{\init}x}) = cost(R_i)$}{
                Add $x$ to $V_A$ and $\mathcal{T}$\\ 
                \If{$x \in X_{\goal}$ and $\cost(\overline{x_{\init}x}) < \cost(\traj)$}{
                    $\traj \gets$ $(\overline{x_{\init}x})$\\
                    $\cost(\traj) \gets \cost(\overline{x_{\init}x'})$
                }
            }
        }

\end{algorithm}

%% file: sections/Analysis.tex
In this section, we establish that \aoalg achieves probabilistic \(\delta\)-robust completeness (Definition \ref{def:delta-robust-completeness}) and asymptotic \(\delta\)-robust near-optimality (Definition \ref{def:asymp_delta_near_opt}). Our analysis contributes a proof technique that reconciles
the covering-ball arguments used to analyze serial
near-optimal planners with the region structure of \aoalg. We begin by defining required notions. 

The \emph{obstacle clearance} of a valid trajectory $\traj$
is the minimum distance from $\traj$ to the invalid set $X_\text{invalid} = X \setminus X_\free$. The \emph{dynamic clearance} of $\traj$
is the maximum distance $\delta_a$ you can displace the start and end points of $\traj$ such that a new, similar (within $\delta_a$ distance from $\traj$) trajectory is feasible according to the dynamics in \eqref{eq:diffEq} 
(see Definition~4 and Lemma~6 in \cite{li2016asymptotically} for details).  
A trajectory is called \textit{$\delta$-robust} if both of its obstacle and dynamic clearances are greater than $\delta$.

\begin{definition}[Probabilistic $\delta$-Robust Completeness]
\label{def:delta-robust-completeness}
Let $\Traj_n^\textsc{alg}$ denote the set of trajectories discovered by an algorithm \textsc{alg} at iteration $n$. Algorithm \textsc{alg} is probabilistically $\delta$-robustly complete, if for every motion planning problem where there exists at least one $\delta$-robust trajectory, the following holds for all independent runs:
\begin{equation*}
    \liminf_{n \to \infty} \;\;\; \mathbb{P}\left( \exists \traj \in \Traj_n^{\textsc{alg}} \text{ s.t. } \traj \in \Traj_\text{sol} \right) = 1.
\end{equation*}
\end{definition}

\begin{definition}[Asymptotic $\delta$-robust Near-Optimality]
    \label{def:asymp_delta_near_opt}

    Consider the setting in Problem~\ref{problem} where there exists at least one $\delta$-robust trajectory. Let $c^*$ denote the minimum achievable cost over all $\delta$-robust solution trajectories. Let $Y_n^{\textsc{alg}}$ denote a random variable representing the minimum cost among all trajectories returned by algorithm $\textsc{alg}$ after iteration $n$. The algorithm $\textsc{alg}$ is asymptotically $\delta$-robust near-optimal if for all independent runs:
    \[
        \mathbb{P}\left(\limsup_{n \to \infty} \;\; Y_n^{\textsc{alg}} \leq h(c^*, \delta)\right) = 1,
    \]
    where $h: \mathbb{R}_{\geq 0} \times \mathbb{R}_{\geq 0} \rightarrow \mathbb{R}_{\geq 0}$  is a function of the optimum cost $c^*$ and the $\delta$ clearance such that $h(c^*, \delta) \geq c^*$.
\end{definition}

In what follows, we prove that \aoalg is both $\delta$-robust complete and asymptotic $\delta$-robust near-optimal.
We leverage the concept of Covering Balls from~\cite{li2016asymptotically}. That is, given a $\delta$, any valid trajectory $\traj$ is covered by a sequence of balls of radius $\delta$.  Let $\mathbb{B}_{r}(x) = \{x' \in \mathbb{R}^n \mid \|x - x'\| \leq \delta \}$ denote a ball of radius $r$ centered at point $x$.

\begin{definition}[Covering Balls]
    \label{def: d_covering_balls}

    For a $\delta$-robust trajectory $\traj: [0, T] \rightarrow X_{\free}$ and a given cost increment $C_{\Delta}> 0$, let $x_0, x_1, \ldots, x_m$ be a sequence of points on $\traj$ such that $\cost(\overline{x_ix_{i+1}}) = C_{\Delta}$ for all $i \in \{0,\ldots, m-1\}$.
    Then, the \emph{covering ball sequence} $$\mathbb{B}(\traj, \delta, C_{\Delta}) = \{\mathbb{B}_\delta(x_0), \ldots, \mathbb{B}_\delta(x_m)\}$$ is a set of $m + 1$ hyper-balls centered at those points with radius $\delta$.
\end{definition}

Theorem 17 in~\cite{li2016asymptotically} guarantees that, using random control sampling, there exists a strictly positive probability $\rho_{\delta} > 0$ of generating a trajectory segment between consecutive balls that remains $\delta$-close to $\traj$. To extend this result to \aoalg, we align the sequence of covering balls with our spatial decomposition $\mathcal{R}$. We identify the sequence of regions corresponding to the optimal trajectory, formally defined as follows.

\begin{definition}[Optimal Trajectory Regions]
    \label{def: d_optimal_trajectory_region}
    Consider a hyper-cube space decomposition $\mathcal{R}$ where each region $\mathcal{R}_i \in \mathcal{R}$ has diagonal length $\delta$. Let $\traj^*$ be a $\delta$-robust optimal trajectory covered by the sequence 
    $\mathbb{B}(\traj^*, \delta, C_{\Delta}) = \{\mathbb{B}_\delta(x^*_0), \ldots, \mathbb{B}_\delta(x^*_m)\}$.
    The Optimal Trajectory Regions is a set of regions $\mathcal{R}^* = \{\mathcal{R}^*_i \in \mathcal{R}\}_{i=0}^m$ such that $\mathcal{R}^*_i$ is the region that contains point $x_i^*$. 
\end{definition}

\begin{figure}[t]
  \centering
  \includegraphics[width=0.6\linewidth]{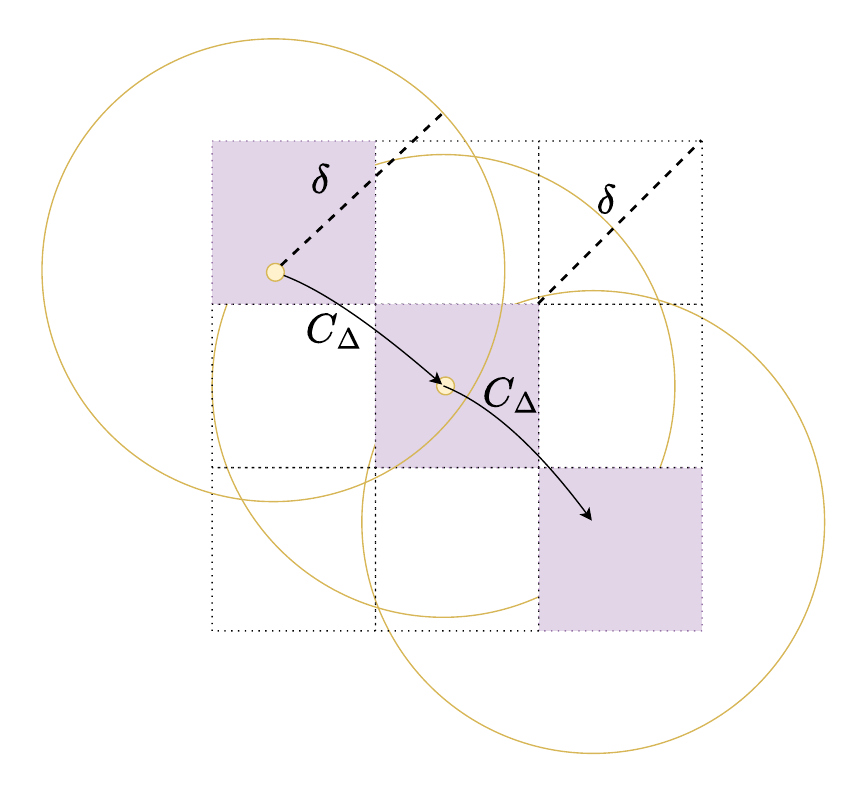}
  \caption{An illustration of a $\delta$-robust optimal trajectory $\traj^*$ (yellow nodes), with nodes equally spaced in terms of cost increments of $C_\Delta$. Each node lies within a region $\mathcal{R}_i$, and the purple regions collectively form the Optimal Trajectory Regions set $\mathcal{R}^*$.}
  \label{fig:R_Star_Illustration}
  \vspace{-1mm}
\end{figure}

Figure~\ref{fig:R_Star_Illustration} depicts an example of Definition~\ref{def: d_optimal_trajectory_region}.
We now show that \aoalg finds trajectory segments between consecutive regions in $\mathcal{R}^*$ with non-zero probability and bounded cost.

\begin{lemma}
    \label{lemma:region_to_region}
    Given a $\delta$-robust optimal trajectory $\traj^*$ with Optimal Trajectory Regions $\mathcal{R}^*$, the probability of \aoalg successfully propagating an arbitrary state $x \in \mathcal{R}_{i-1}^*$ to a state $x' \in \mathcal{R}_{i}^*$ such that the cost of the segment $\overline{x x'}$ satisfies
    \begin{align*}
        \cost(\overline{x x'}) \leq C_{\Delta} + K_c \delta
    \end{align*}
    is bounded below by a positive constant $\rho_{\mathcal{R}} > 0$.
\end{lemma}

\begin{proof} For the case where $R_{i-1}^* = R^*_{i}$, we have that $x = x'$, and the result trivially holds. We now consider the case when $R_{i-1}^* \neq R^*_{i}$. From Theorem~17 in \cite{li2016asymptotically}, for any state in the covering ball $\mathbb{B}_{i-1}$ and using random control sampling, there exists a strictly positive probability $\rho_\delta$ of successfully generating a trajectory segment to any state within $\mathbb{B}_i$ that is $\delta$-similar to the optimal segment. By definition, the diagonal length of region $\mathcal{R}^*_{i-1}$ is $\delta$, implying $\mathcal{R}^*_{i-1} \subseteq \mathbb{B}_{i-1}$ (since $\mathbb{B}_{i-1}$ has radius $\delta$). This implies that any valid node existing in $\mathcal{R}^*_{i-1}$ is in the required covering ball $\mathbb{B}_{i-1}$.  Similarly, $\mathcal{R}^*_j \subseteq \mathbb{B}_i$. Thus, any state in $\mathcal{R}^*_{i-1}$ can transition to states in $\mathcal{R}^*_i$ with probability $\rho_{\mathcal{R}} > 0$. Since this transition is $\delta$-similar to the optimal segment which has cost $C_{\Delta}$, the Lipschitz continuity of the cost function guarantees
    \begin{align*}
        |\cost(\overline{x x'}) - C_{\Delta}| \leq K_c \delta \quad \implies \quad \cost(\overline{xx'}) \leq C_{\Delta} + K_c \delta. 
    \end{align*}
\end{proof}

Now, we present our main analytical result.

\begin{theorem}
    \label{theorem:near-optimal}
    \aoalg is asymptotically $\delta$-robustly near-optimal. 
\end{theorem}
\begin{proof}
    %
    Consider $\delta$-robust optimal trajectory $\traj^*$ with cost $c^* = \cost(\traj^*)$ and its corresponding optimal trajectory regions $\mathcal{R}^* = \{\mathcal{R}^*_j\}_{j=0}^{m}$ as defined in Definition~\ref{def: d_optimal_trajectory_region}.
    We proceed by induction to show that \aoalg generates a path along the Optimal Trajectory Regions $\mathcal{R}^*_j$ with accumulated cost $C_j \leq (C_{\Delta} + K_c \delta)\cdot j$. We first define the events:
    \begin{enumerate}
        \item $A_j^{(n)}$: On iteration $n$, \aoalg generates a trajectory $\traj_j$ terminating in region $\mathcal{R}^*_j$ with cost satisfying $\cost(\traj_j) \leq (C_{\Delta} + K_c \delta)\cdot j$.
        \item $E_j^{(n)}$: By iteration $n$, \aoalg has generated at least one trajectory satisfying event $A_j^{(n)}$.
    \end{enumerate}

    Base Case ($E_0^{(n)}$):
    The root node $x_0 = x^*_0 \in \mathcal{R}^*_0$ has cost 0, satisfying the bound for $j=0$. Thus $P(E_0^{(n)}) = 1$ for all $n \geq 1$.
    
    Inductive Step (from $E_{j-1}^{(n)}$ to $E_j^{(n)}$):
    Assume event $E_{j-1}^{(n)}$ holds. \aoalg has a node in $\mathcal{R}^*_{j-1}$ with cost $C_{j-1} \leq (j-1)(C_{\Delta} + K_c \delta)$. As $n \to \infty$, this node is selected for expansion infinitely often. By Lemma~\ref{lemma:region_to_region}, there is a positive probability $\rho_{\mathcal{R}}$ of generating a trajectory segment from $\mathcal{R}^*_{j-1}$ to $\mathcal{R}^*_j$ with cost bounded above by $C_{\Delta} + K_c \delta$.
    Then, two mutually exclusive scenarios arise:
    
    Case 1 (Acceptance): \aoalg directly accepts trajectory $\traj_j$ (or $\traj_{j-1} = \traj_j$), as it achieves the lowest cost for its terminal region $\mathcal{R}^*_j$. This occurs with strictly positive probability $\gamma_1 > 0$.
    
    Case 2 (Rejection): \aoalg rejects trajectory $\traj_j$ due to the existence of another trajectory $\traj'$ with a lower cost endpoint node already in region $\mathcal{R}^*_j$. Since the rejected trajectory $\traj_j$ is $\delta$-optimal, the accepted trajectory $\traj'$ must also be $\delta$-optimal. Thus, even rejection results in having a $\delta$-optimal trajectory in region $\mathcal{R}^*_j$, and this case occurs with strictly positive probability $\gamma_2 > 0$.

    In either case, the region $\mathcal{R}^*_j$ is guaranteed to contain a node satisfying the $\delta$-bounded cost. Thus, $\lim_{n \to \infty} \mathbb{P}(E_j^{(n)}) = \gamma_1 + \gamma_2 = 1$. By induction, the event $E_j^{(n)}$ holds for all segment indices $j$. We now derive the cost bound. From Lemma~\ref{lemma:region_to_region}, each segment contributes an additive error of $K_c \delta$. Thus, for the $j$-th region, the accumulated cost of the trajectory $\traj_j$ satisfies:
    \begin{align*}
        \cost(\traj_j) &\leq \cost(\traj^*_j) + j \cdot K_c \cdot \delta = j \cdot C_{\Delta} + j \cdot K_c \cdot \delta.
    \end{align*}
    Let $m = \lceil c^*/C_{\Delta} \rceil$ be the index of the final region containing the goal state. We have that
    \begin{align*}
        \cost(\traj_m) &\leq c^* + \frac{c^*}{C_{\Delta}} \cdot K_c \cdot \delta = c^* \left( 1 + \frac{K_c \cdot \delta}{C_{\Delta}} \right).
    \end{align*}

    Let $Y^{\textsc{alg}}_n$ denote the cost of the solution found by \aoalg at iteration $n$. The probability that the algorithm finds a solution satisfying this bound corresponds to the probability of the final event $E_m^{(n)}$:
    \begin{align*}
        \mathbb{P}\left( Y^{\textsc{alg}}_n \leq c^* \left( 1 + \frac{K_c \delta}{C_{\Delta}} \right) \right) = \mathbb{P}(E_m^{(n)}).
    \end{align*}
    As established in the inductive step, since the transition probability $\rho_{\mathcal{R}} > 0$, the event $E_m^{(n)}$ occurs almost surely as $n \to \infty$:
    \begin{align*}
        \mathbb{P}\left( \limsup_{n \to \infty} \;\; Y^{\textsc{alg}}_n \leq c^* \left( 1 + \frac{K_c \delta}{C_{\Delta}} \right) \right) = \lim_{n \to \infty} \mathbb{P}(E_m^{(n)}) = 1.
    \end{align*}
\end{proof}

\begin{corollary}
    \aoalg solves Problem~\ref{problem} with a near-optimality parameter $\beta > 0$ (see Definition~\ref{def:near optimal}) if $ \delta \leq \frac{C_{\Delta}}{K_c} \beta$.
\end{corollary}

\begin{remark}
    Explicitly computing the required (diagonal) $\delta$ for a specific $\beta$ is generally infeasible, as both the Lipschitz constant $K_c$ and the segmentation cost $C_{\Delta}$ of the optimal trajectory are unknown for general non-linear systems. This limitation is inherent to the analysis of asymptotically near-optimal planners such as SST~\cite{li2016asymptotically}. In practice, this theoretical result serves as a qualitative guideline rather than a quantitative rule: the term $\frac{\delta}{C_{\Delta}}$ represents the relative decomposition-based error of the algorithm. Since $\delta$ is a tunable parameter, the approximation error can be made arbitrarily small ($\lim_{\delta \rightarrow 0} 1 + \frac{K_c \delta}{C_{\Delta}} = 1$), allowing \aoalg to approach optimality up to any arbitrary precision.
\end{remark}

\begin{remark}[Non-Uniform and Adaptive Decompositions]
The analysis carries over to non-uniform decompositions without losing asymptotic near-optimality, provided any region along an optimal trajectory has diameter no greater than the local clearance. It likewise supports adaptive strategies in which $\delta$ is decreased over time to refine solution quality (analogous to SST$^*$~\cite{li2016asymptotically}). A full analysis of partitioning strategies is deferred to future
work.
\end{remark}

From Theorem~\ref{theorem:near-optimal}, we establish that \aoalg is probabilistically $\delta$-robust complete.

\begin{theorem}
    \label{theorem:delta-robust-complete}
    \aoalg is probabilistically $\delta$-robust complete.
\end{theorem}

\begin{proof}
    This result follows directly from Theorem~\ref{theorem:near-optimal}.
\end{proof}

%% file: sections/experiments.tex
We evaluate \aoalg's performance across four 3D environments using three dynamical systems. Three environments are taken from \cite{Perrault2025kinopax} (Fig.~\ref{fig:environments}a-c), while a fourth environment (Fig.~\ref{fig:environments}d) is designed to test \aoalg's capability for planning in tight corridors over extended horizons. The robot dynamics considered are: (i) a 6D double integrator, (ii) a 6D Dubins airplane~\cite{chitsaz2007time}, and (iii) a 12D nonlinear quadcopter~\cite{etkin1995dynamics}.

\begin{figure}[t]
    \centering
     \begin{subfigure}[b]{0.45\columnwidth}
        \centering
        \includegraphics[width=\textwidth]{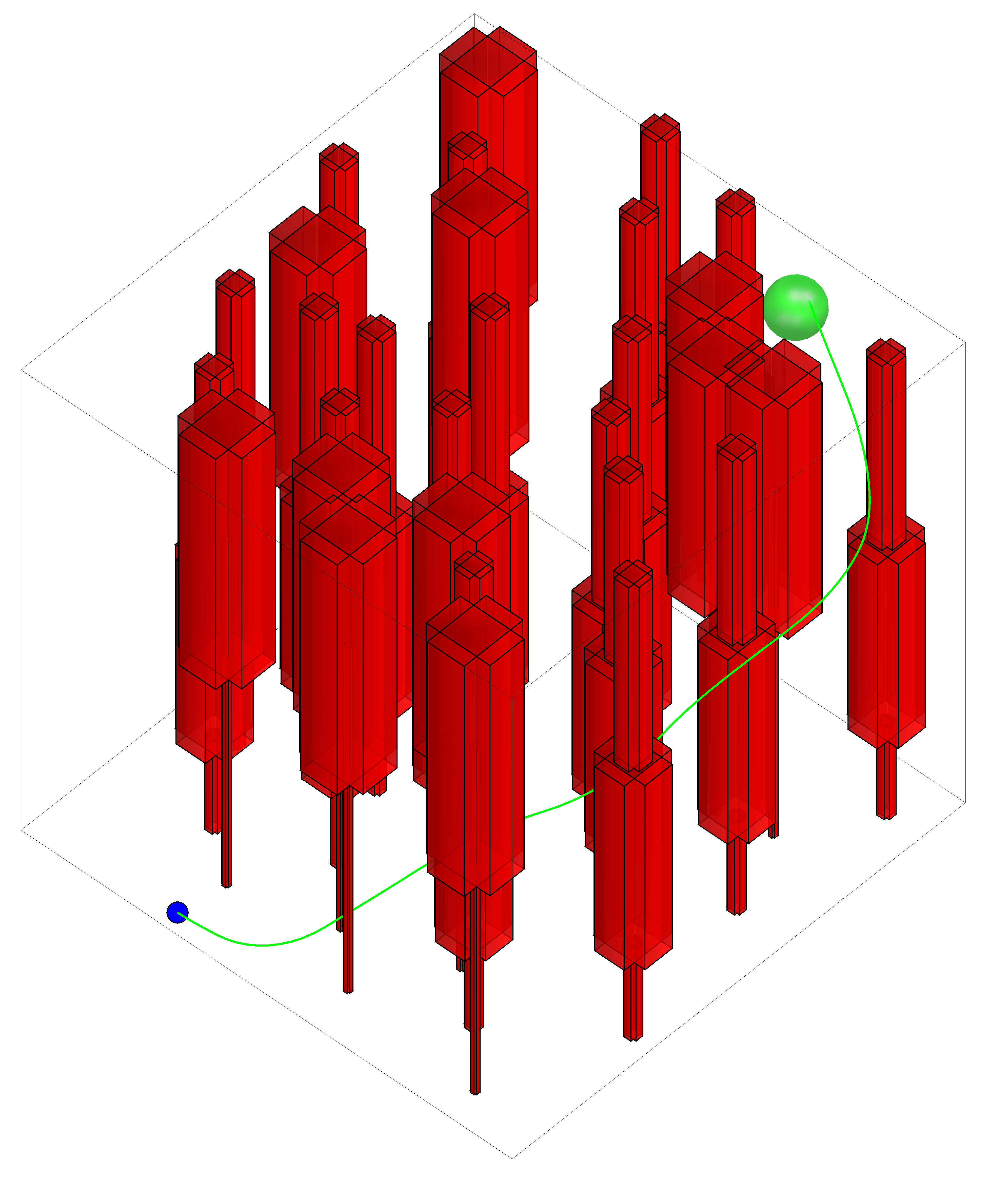}
        \caption{Forest}
        \label{fig:trees}
    \end{subfigure}
    ~~
    \begin{subfigure}[b]{0.45\columnwidth}
        \centering
        \includegraphics[width=\textwidth]{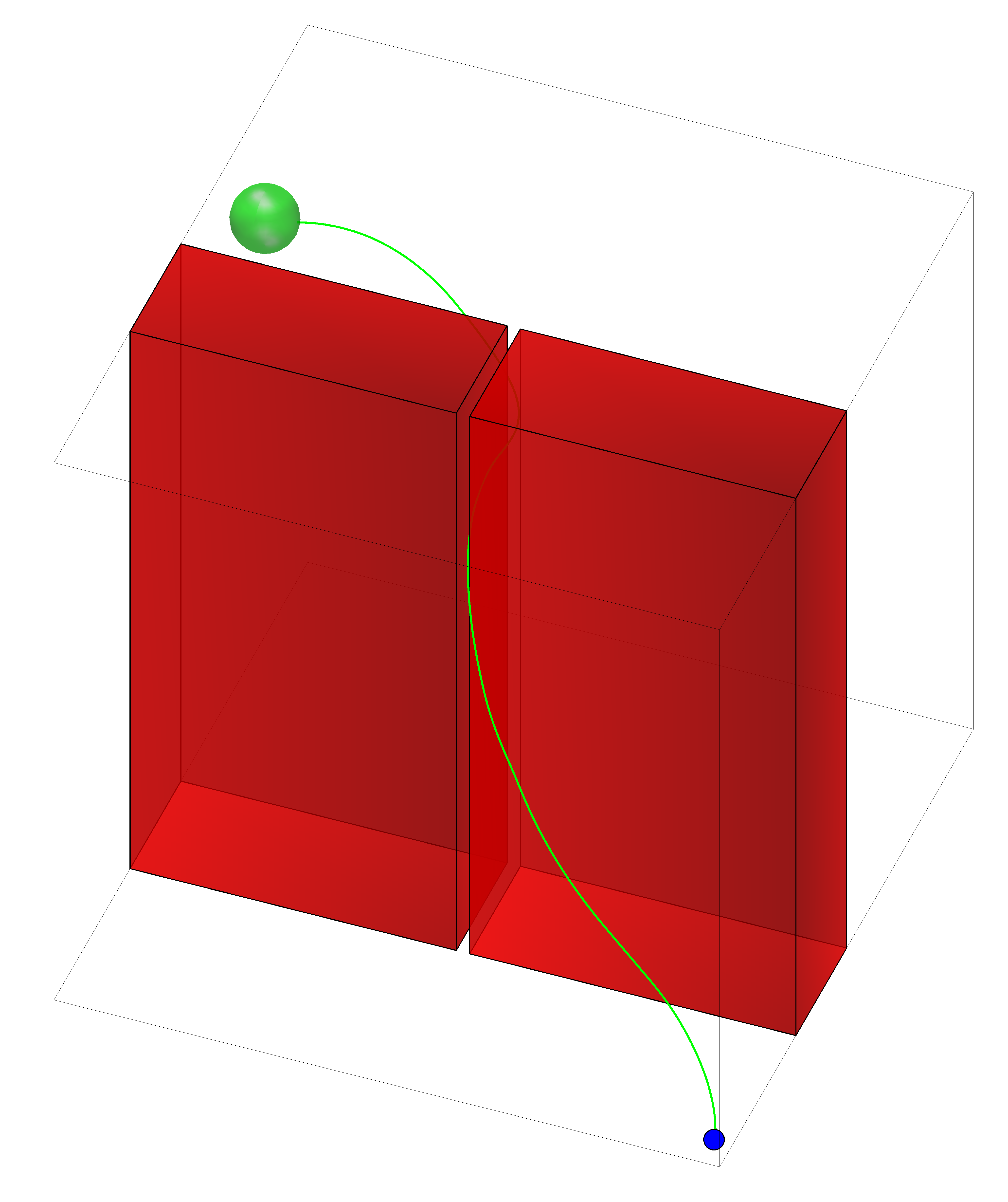}
        \caption{Narrow}
        \label{fig:narrowPassage}
    \end{subfigure}
     
    \begin{subfigure}[b]{0.45\columnwidth}
        \centering
        \includegraphics[width=\textwidth]{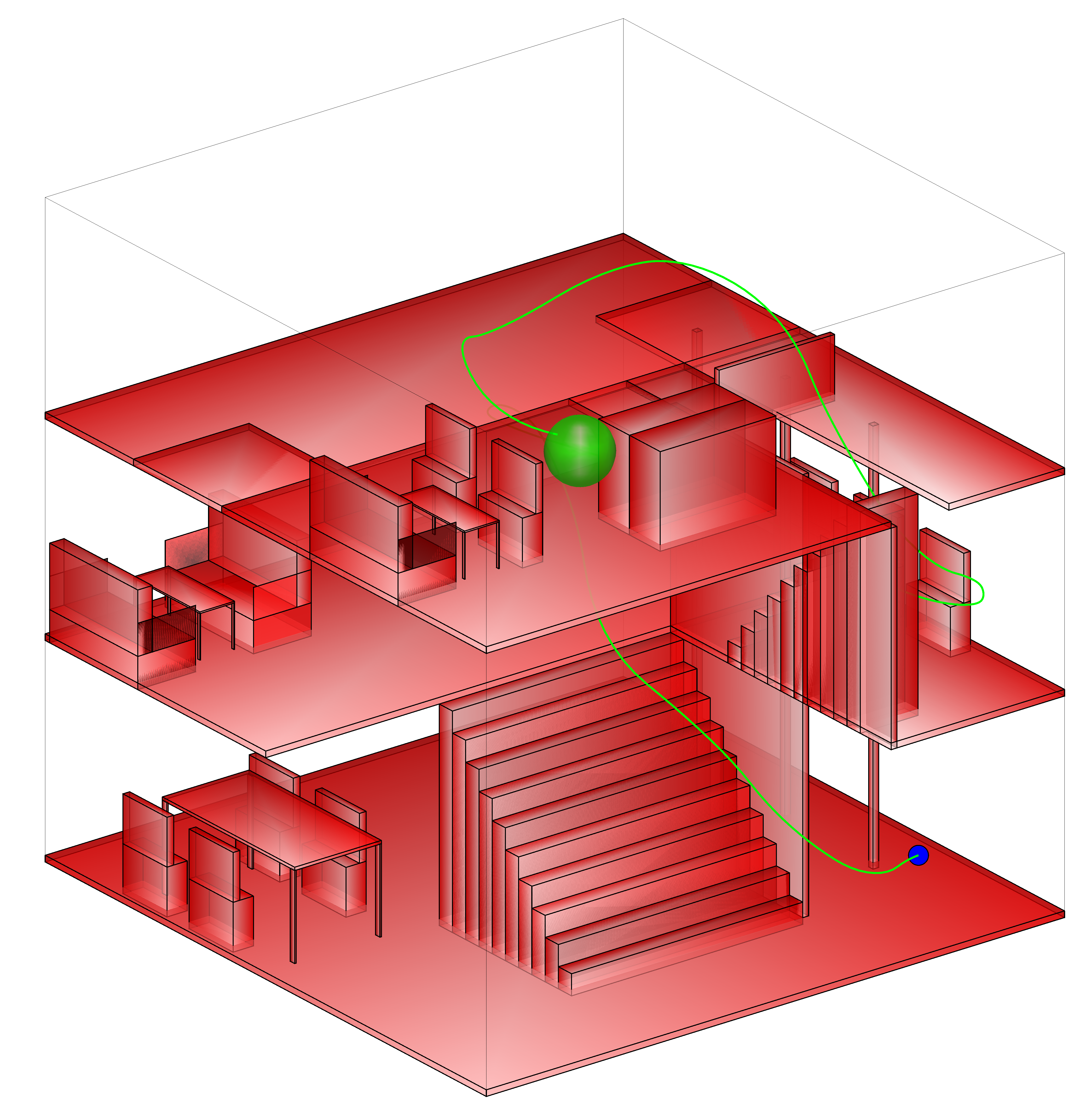}
        \caption{House}
        \label{fig:house}
    \end{subfigure}
    ~~
    \begin{subfigure}[b]{0.45\columnwidth}
        \centering
        \includegraphics[width=\textwidth]{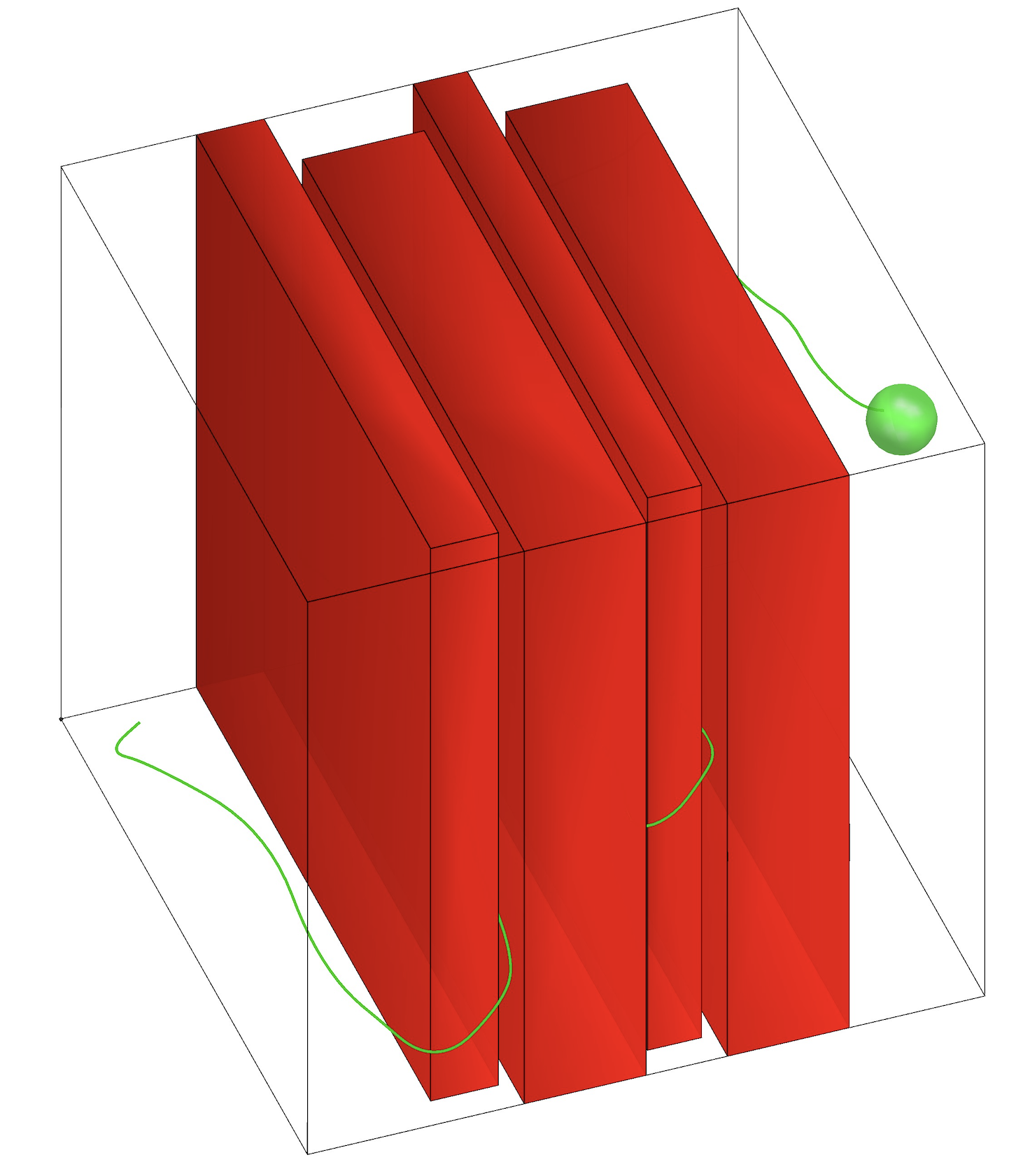}
        \caption{Zig Zag}
        \label{fig:zigzag}
    \end{subfigure}
    \caption{Considered Environments. 
    }
    \label{fig:environments}
\end{figure}

We implemented \aoalg in CUDA C and benchmarked it alongside \alg on an NVIDIA RTX 4090 GPU with 16,384 CUDA cores and 24\,GB of RAM, as well as embedded Jetson Orin Nano GPU. The serial SBMP SST, implemented in C++ using OMPL~\cite{sucan2012the-open-motion-planning-library}, ran on an Intel i9-14900K CPU (base clock 4.4\,GHz, 128\,GB RAM). Implementation details and choices of hyper-parameters are provided in Appendix~\ref{appendix: hyperparameters}. 

For each dynamical system and environment, we compare \aoalg against \alg~\cite{Perrault2025kinopax} and SST~\cite{li2016asymptotically}, assessing time to first solution relative to \alg and solution quality relative to SST.

To highlight how the choice of decomposition size affects the performance of \aoalg, we use two configurations of \aoalg with extreme values of $\delta$ in each experiment: 
(i) \aoalg{}\textit{-large-$\delta$} uses a very coarse decomposition and small expected tree size $t_e$\footnote{$t_e$ is a parameter 
that
defines the maximum number of nodes in GPU memory and directly affects branching factor $\lambda$; see Appendix~\ref{appendix: hyperparameters} for details.}, emphasizing rapid first solutions, and
(ii) \aoalg{}\textit{-small-$\delta$} uses a fine decomposition, 
prioritizing long-horizon cost improvement (see Appendix~\ref{appendix: empirical eval details} for details). These configurations characterize how $\delta$ trades first-solution latency against convergence quality. 

Finally, to assess deployability on embedded hardware, we evaluate \aoalg on an embedded Jetson Orin Nano GPU across a range of intermediate $\delta$ values and compare against \alg in two of the environments.

\subsection{Benchmark Results}

Tables~\ref{table:AO_KPAX_Results_Env1}-\ref{table:AO_KPAX_Results_Env3}
summarize results, 
presenting median first solution time, normalized median first solution cost, median final solution time, normalized median final solution cost, 
and the success rate across 100 queries. 

Table~\ref{table:AO_KPAX_Results_Env1} shows that, for the 6‑D double‑integrator, \aoalg{}\textit{-large-$\delta$} reaches a first solution in under $11 \, \mathrm{ms}$ across all environments, comparable to \alg, which does so in under $8\,\mathrm{ms}$. A comparison of their first solution times and solution quality over time is shown in Figure~\ref{fig:box_KPAX_AOKPAX}. Relative to SST, \aoalg{}\textit{-large-$\delta$} finds the first solution $650\times$, $1600\times$, $1100\times$, and $3300\times$ faster in Environments~a–d, respectively. On average, \aoalg{}\textit{-large-$\delta$}'s first solution cost is $65\%$ of SST’s first‑solution cost. After $10\,\mathrm{ms}$ of planning, \aoalg{}\textit{-large-$\delta$} converges to a local optimum with an average cost that is $61\%$ of SST’s first solution cost. In comparison, the fine-grained hyperparameter variant, \aoalg{}\textit{-small-$\delta$}, converges to a local optimum after an average of $160\,\mathrm{s}$ of planning, achieving an average cost that is $55\%$ of SST’s initial cost. Meanwhile, SST, after the full five‑minute duration, improves only to $84\%$ of its own initial cost. A visualization comparing SST with \aoalg{}\textit{-large-$\delta$} and \aoalg{}\textit{-small-$\delta$} is shown in Figure~\ref{fig:box_KPAX_AOKPAX_SST}.

\begin{table}[t]
    \centering
    \caption{
    Benchmark results over 100 trials with a 300-second maximum planning time for the 6D double integrator system.
    }
    \label{table:AO_KPAX_Results_Env1}
    \begin{tabular}{l rrrr r}
        \toprule
                                   \multirow{2}{*}{Algorithm}   & 
                                   \multicolumn{2}{c}{\underline{\hspace{4mm}First Sol.\hspace{4mm}}}
                                   &
                                   \multicolumn{2}{c}{\underline{\hspace{5mm}Final Sol.\hspace{4mm}}}
                                   & \multirow{2}{*}{Succ.\%}\\
                                   & $t$ (ms) & $\cost$ & $t$ (ms) & $\cost$ &  \\ 
        \hline
        \multicolumn{6}{c}{Environment a} \\ 
        \hline
        SST                 & 1950.0 & 1.00 & 165201.5 & 0.79 & 100 \\
        \alg              & 3.7 & 0.70 & 4.0 & 0.70 & 100 \\
        \aoalg{}\textit{-large-$\delta$}              & 3.0 & 0.72 & 6.2 & 0.68 & 100 \\
        \aoalg{}\textit{-small-$\delta$}     & 1173.0 & 0.68 & 180017.5 & 0.64 & 100 \\
        \hline
        \multicolumn{6}{c}{Environment b} \\ 
        \hline
        SST                   & 3525.0 & 1.00 & 97060.0 & 0.78 & 100 \\
        \alg               & 2.8 & 0.50 & 2.8 & 0.49 & 100 \\
        \aoalg{}\textit{-large-$\delta$}                 & 2.2 & 0.50 & 5.3 & 0.47 & 100\\
        \aoalg{}\textit{-small-$\delta$}      & 1002.2 & 0.48 & 252173.0 & 0.44 & 100 \\
        \hline
        \multicolumn{6}{c}{Environment c} \\ 
        \hline
        SST                   & 5225.5 & 1.00 & 131089.5 & 0.84 & 100 \\
        \alg             & 5.1 & 0.62 & 5.4 & 0.62 & 100 \\
        \aoalg{}\textit{-large-$\delta$}              & 4.7 & 0.62 & 8.6 & 0.56 & 100 \\
        \aoalg{}\textit{-small-$\delta$}     & 1556.0 & 0.57 & 133854.0 & 0.48 & 100 \\
        \hline
        \multicolumn{6}{c}{Environment d} \\ 
        \hline
        SST                     & 34760.0 & 1.00 & 146466.5 & 0.95 & 100 \\
        \alg              & 7.7 & 0.83 & 8.5 & 0.83 & 100 \\
        \aoalg{}\textit{-large-$\delta$}               & 10.6 & 0.79 & 19.9 & 0.73 & 100 \\
        \aoalg{}\textit{-small-$\delta$}     & 2698.1 & 0.72 & 91483.6 & 0.63 & 100 \\
        \hline
    \end{tabular}
\end{table}

\begin{figure}[t]
    \centering
    \includegraphics[width=0.9\columnwidth]{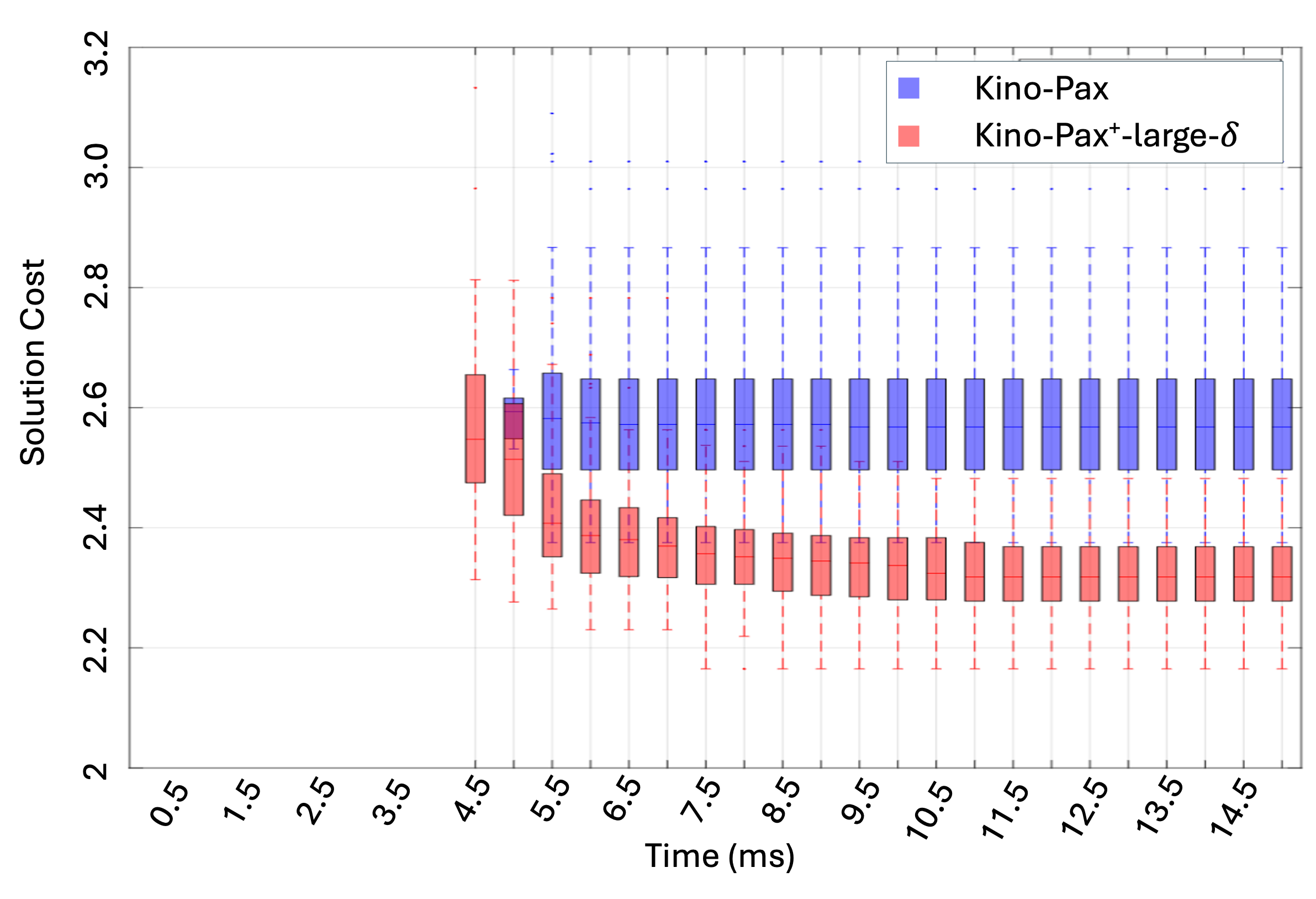}
    \caption{Planning performance of \alg and \aoalg{}\textit{-large-$\delta$} on the 6D double integrator in Environment~c over a $15\,\mathrm{ms}$ planning duration.}
    \label{fig:box_KPAX_AOKPAX}
\end{figure}

\begin{figure}[ht]
    \centering
    \includegraphics[width=0.95\columnwidth]{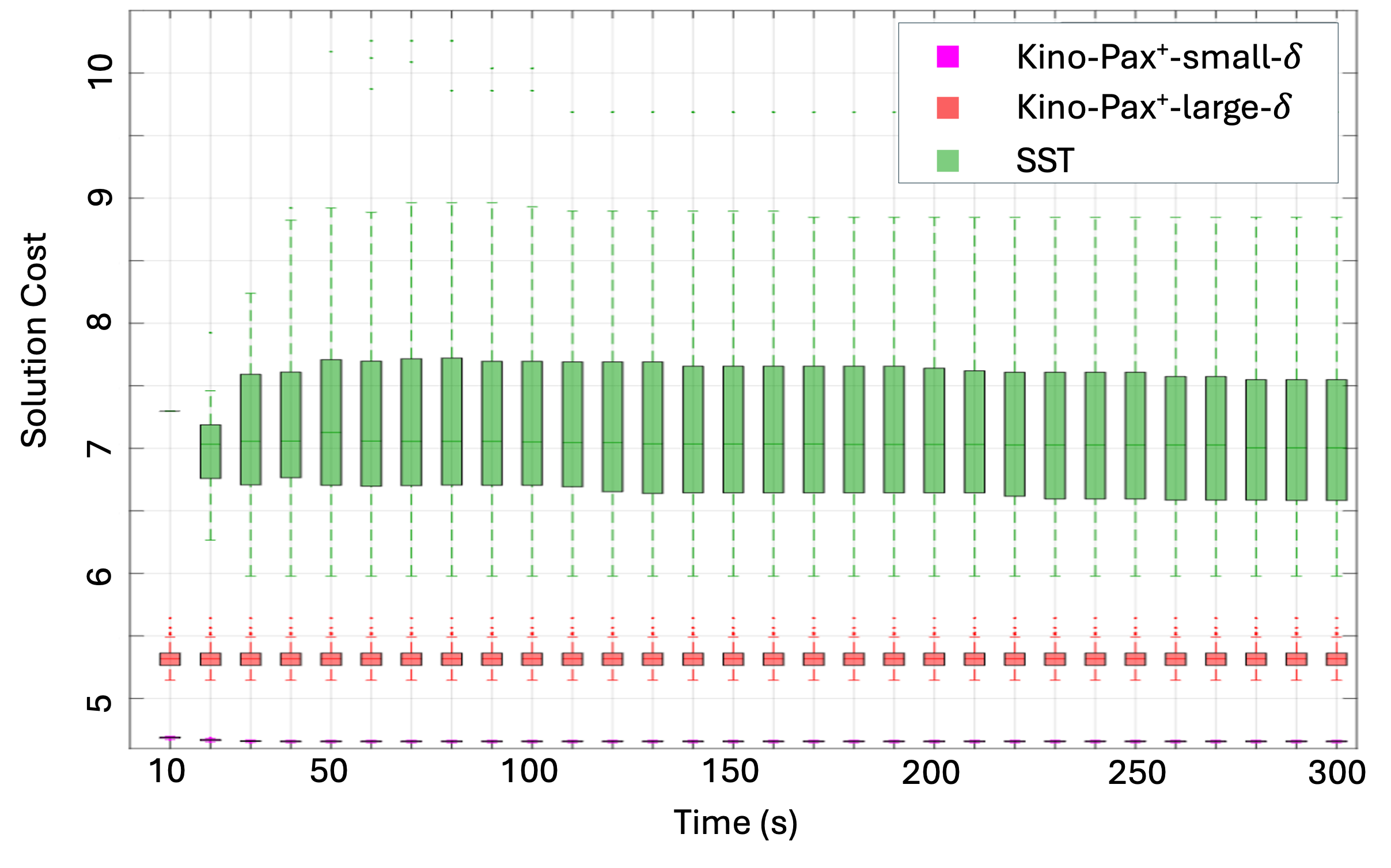}
    \caption{Performance of \aoalg (both variants) and SST on 6D double integrator in the zig-zag Environment d
    over a $300\,\mathrm{s}$ planning time.}
    \label{fig:box_KPAX_AOKPAX_SST}
\end{figure}

\begin{table}[t]
    \centering
    \caption{
    Benchmark results over 100 trials with a 300-second maximum planning time for the Dubins Airplane system.
    }
    \label{table:AO_KPAX_Results_Env2}
    \begin{tabular}{l r  rrrrr }
        \toprule
                            \multirow{2}{*}{Algorithm}   & 
                                   \multicolumn{2}{c}{\underline{\hspace{5mm}First Sol.\hspace{4mm}}}
                                   &
                                   \multicolumn{2}{c}{\underline{\hspace{5mm}Final Sol.\hspace{4mm}}}
                                   & \multirow{2}{*}{Succ.\%}\\
                                   & $t$ (ms) & $\cost$ & $t$ (ms) & $\cost$ &  \\ 
        \hline
        \multicolumn{6}{c}{Environment a} \\ 
        \hline
        SST                      & 10341.0 & 1.00 & 182846.0 & 0.91 & 100 \\
        \alg              & 4.2 & 0.86 & 4.2 & 0.86 & 100 \\
        \aoalg{}\textit{-large-$\delta$}              & 6.1 & 0.87 & 14.5 & 0.77 & 100 \\
        \aoalg{}\textit{-small-$\delta$}      & 1559.7 & 0.81 & 89458.9 & 0.74 & 100 \\
        \hline
        \multicolumn{6}{c}{Environment b} \\ 
        \hline
        SST                    & 24294.5 & 1.00 & 141664.0 & 0.83 & 100 \\
        \alg                   & 3.2 & 0.67 & 3.2 & 0.67 & 100 \\
        \aoalg{}\textit{-large-$\delta$}     & 4.7 & 0.65 & 12.0 & 0.61 & 100 \\
        \aoalg{}\textit{-small-$\delta$}     & 1419.5 & 0.62 & 56694.8 & 0.58 & 100 \\
        \hline
        \multicolumn{6}{c}{Environment c} \\ 
        \hline
        SST                     & 154149.0 & 1.00 & 218667.0 & 0.97 & 87 \\
        \alg               & 7.5 & 0.94 & 7.5 & 0.94 & 100 \\
        \aoalg{}\textit{-large-$\delta$}            & 13.4 & 0.91 & 34.8 & 0.80 & 100 \\
        \aoalg{}\textit{-small-$\delta$}      & 2274.8 & 0.75 & 250972.0 & 0.64 & 100 \\
        \hline
        \multicolumn{6}{c}{Environment d} \\ 
        \hline
        SST              & NaN & NaN & NaN & NaN & 0 \\
        \alg           & NaN & NaN & NaN & NaN & 0 \\
        \aoalg{}\textit{-large-$\delta$}         & NaN & NaN & NaN & NaN & 0 \\
        \aoalg{}\textit{-small-$\delta$}     & 4320.7 & 1.00 & 71019.9 & 0.91 & 100 \\
        \hline
    \end{tabular}
\end{table}

For the Dubins Airplane system in Table~\ref{table:AO_KPAX_Results_Env2}, we see largely similar results, where \aoalg{}\textit{-large-$\delta$} achieves first solutions on average $6000\times$ faster than SST, with first solution costs that are $78\%$ of SST’s first solution cost. \aoalg{}\textit{-large-$\delta$} converges to solutions that are $72\%$ of SST’s first solution cost, while \aoalg{}\textit{-small-$\delta$} converges to solutions that are $65\%$ of the cost. Notably, in Env.~d, only \aoalg{}\textit{-small-$\delta$} is able to find a valid solution within the planning time. This environment posed a particularly challenging problem due to the Dubins Airplane system’s large turning radius and the presence of long, narrow passages with tight turns. Due to its fine decomposition and large expected tree size, 
\aoalg{}\textit{-small-$\delta$} is able to find an first solution in $4.3$ seconds and improve that solution by an additional $9\%$ over the five-minute planning period.

\begin{table}[t]
    \centering
    \caption{
    Benchmark results over 100 trials with a 300-second maximum planning time for the 12D Nonlinear Quadcopter system.
    }
    \label{table:AO_KPAX_Results_Env3}
    \begin{tabular}{l r   rrrr }
        \toprule
                                   \multirow{2}{*}{Algorithm}   & 
                                   \multicolumn{2}{c}{\underline{\hspace{4mm}First Sol.\hspace{4mm}}}
                                   &
                                   \multicolumn{2}{c}{\underline{\hspace{5mm}Final Sol.\hspace{4mm}}}
                                   & \multirow{2}{*}{Succ.\%}\\
                                   & $t$ (ms) & $\cost$ & $t$ (ms) & $\cost$ &  \\ 
        \hline
        \multicolumn{6}{c}{Environment a} \\ 
        \hline
        SST         & 86931.0 & 1.00 & 218315.0 & 0.92 & 84 \\
        \alg           & 83.7 & 0.72 & 98.2 & 0.68 & 100 \\
        \aoalg{}\textit{-large-$\delta$}            & 159.3 & 0.72 & 1764.5 &  0.65 & 100 \\
        \aoalg{}\textit{-small-$\delta$}      & 4214.2 & 0.69 & 176232.0 & 0.62 & 100 \\
        \hline
        \multicolumn{6}{c}{Environment b} \\ 
        \hline
        SST                 & 132363.0 & 1.00 & 221516.0 & 0.96 & 53 \\
        \alg                 & 88.8 & 0.55 & 99.8 & 0.52 & 100 \\
        \aoalg{}\textit{-large-$\delta$}   & 145.3 & 0.55 & 1396.9 & 0.48 & 100 \\
        \aoalg{}\textit{-small-$\delta$}      & 4015.4 & 0.51 & 185098.5 & 0.45 & 100 \\
        \hline
        \multicolumn{6}{c}{Environment c} \\ 
        \hline
        SST                     & 209123.0 & 1.00 & 271112.5 & 0.93 & 42 \\
        \alg           & 147.6 & 0.65 & 163.1 & 0.62 & 100 \\
        \aoalg{}\textit{-large-$\delta$}                & 263.0 &  0.62 & 1406.6 & 0.53 & 100 \\
        \aoalg{}\textit{-small-$\delta$}      & 5631.6 & 0.59 & 246833.0 & 0.47 & 100 \\
        \hline
        \multicolumn{6}{c}{Environment d} \\ 
        \hline
        SST                       & NaN & NaN & NaN & NaN & 0 \\
        \alg           & 287.7 & 0.98 & 303.4 &  0.97 & 100 \\
        \aoalg{}\textit{-large-$\delta$}             & 1038.1 &  1.00 & 2918.0 & 0.88 & 100 \\
        \aoalg{}\textit{-small-$\delta$}      & 11504.0 & 0.93 & 221927.0 & 0.78 & 100 \\
        \hline
    \end{tabular}
\end{table}

\begin{figure}[ht]
    \centering
    \includegraphics[width=0.9\columnwidth]{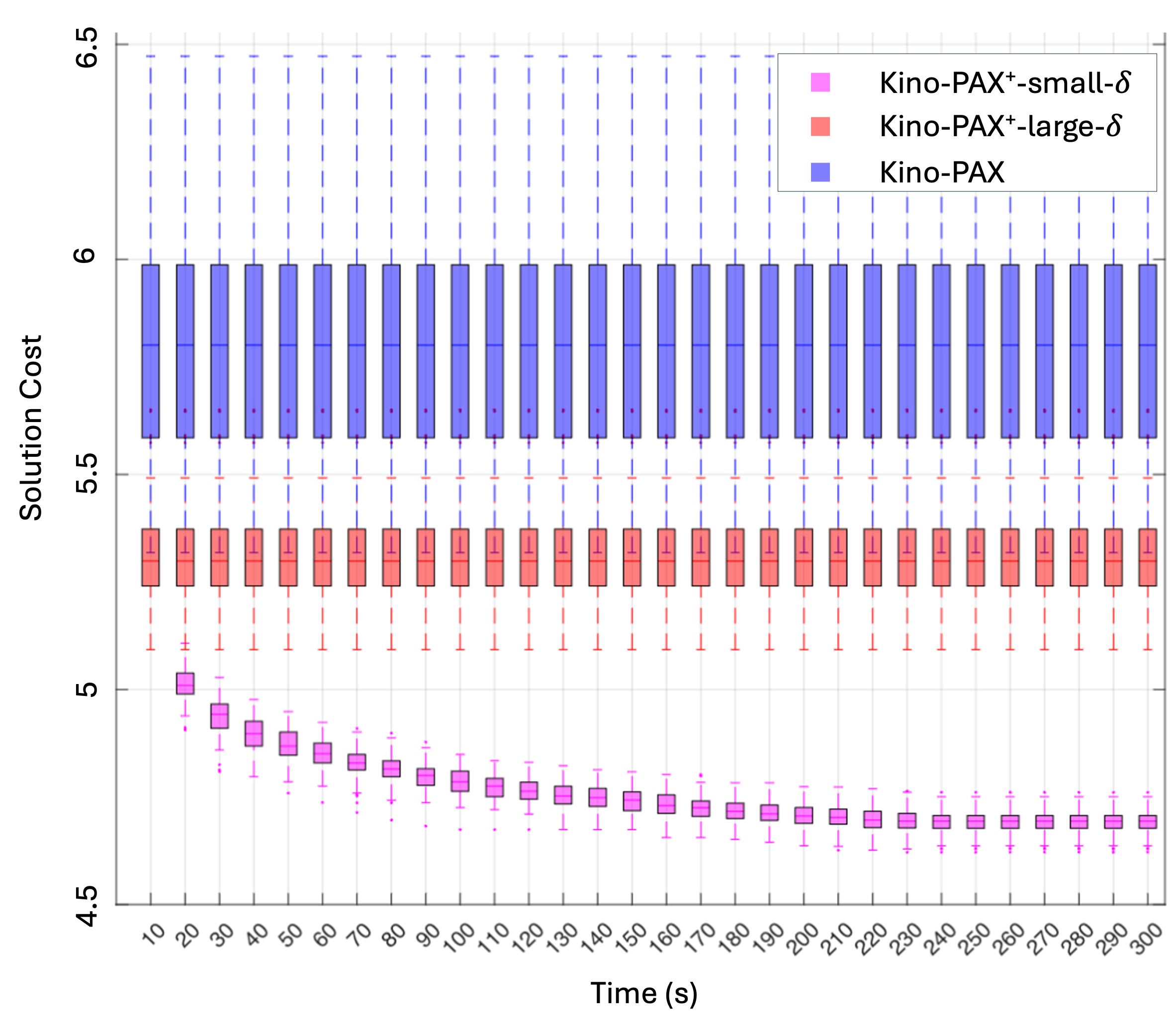}
    \caption{Planning performance of \aoalg{}\textit{-small-$\delta$}, \aoalg{}\textit{-large-$\delta$}, and \alg on the 12D nonlinear quadcopter system in the zig-zag environment d. over a $300\,\mathrm{s}$ planning duration.}
    \label{fig:box_KPAX_AOKPAX_AOKPAX}
\end{figure}

Lastly, Table~\ref{table:AO_KPAX_Results_Env3} and Figure~\ref{fig:box_KPAX_AOKPAX_AOKPAX} present the planning results for the 12D nonlinear quadcopter system. Both tuning variants of \aoalg consistently outperform SST across all evaluated metrics, including time to first solution, first solution cost, and final solution cost. In these experiments, SST successfully found solutions in only $84\%$, $53\%$, $42\%$, and $0\%$ of trials in environments~a–d, respectively. In contrast, both \aoalg{}\textit{-large-$\delta$} and \aoalg{}\textit{-small-$\delta$} achieved a $100\%$ success rate across all environments. On average, \aoalg{}\textit{-large-$\delta$} finds the first solution $750\times$ faster than SST and with a cost that is $63\%$ of SST’s first solution cost. It converges to a final solution that is $55\%$ of SST’s first solution cost, while \aoalg{}\textit{-small-$\delta$} converges to a final solution that is $51\%$ of SST’s. In comparison, SST converges to a final solution that is $94\%$ of its own first solution cost.

Overall, \aoalg outperforms existing serial asymptotically
near-optimal planners in \emph{all} assessed metrics. Additionally, \aoalg{}\textit{-small-$\delta$} shows that the algorithm remains effective and reliable for extremely challenging problems. Finally, unlike \alg, \aoalg is consistently able to optimize towards the lowest final costs across all environments.

\subsection{Embedded GPU and Sensitivity to $\delta$}

\begin{figure*}
    \centering
    \begin{subfigure}[t]{.37\textwidth}
        \includegraphics[width=\linewidth]{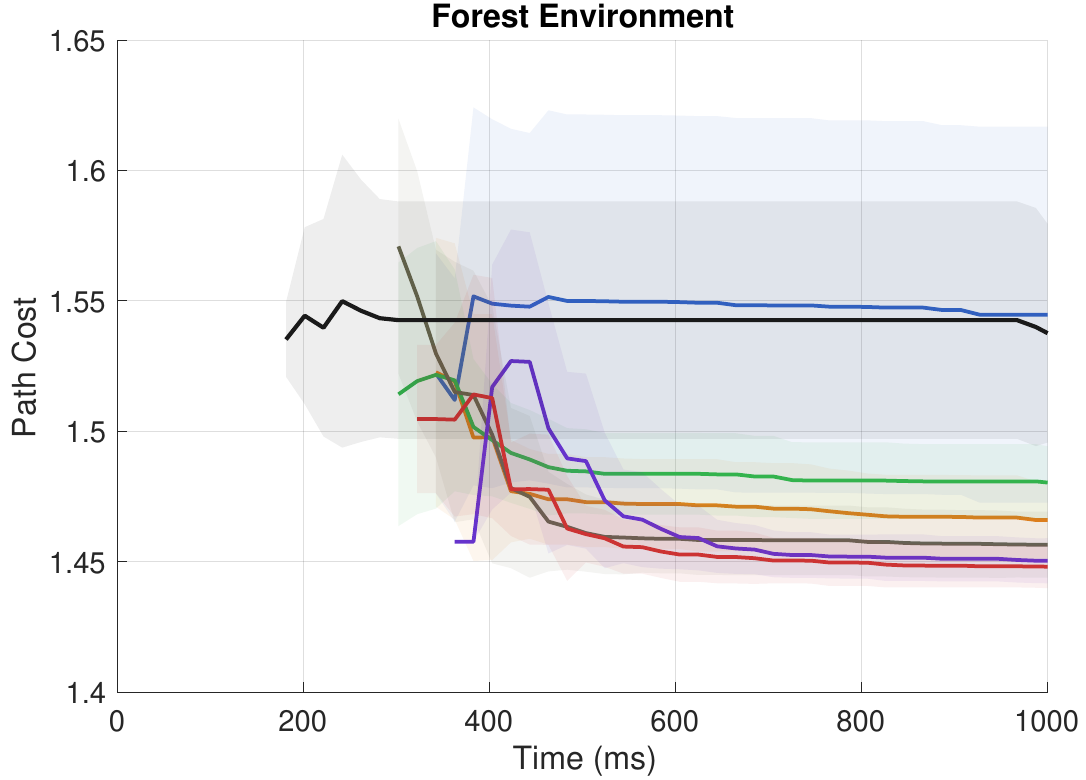}
    %
    ~~~~~
    \end{subfigure}
    \begin{subfigure}[t]{.16\textwidth}
        \vspace{-38mm}
        \includegraphics[width=\linewidth]{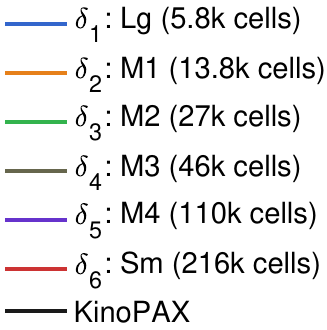}
    ~~~~~
    \end{subfigure}
    \begin{subfigure}[t]{.37\textwidth}
        \includegraphics[width=\linewidth]{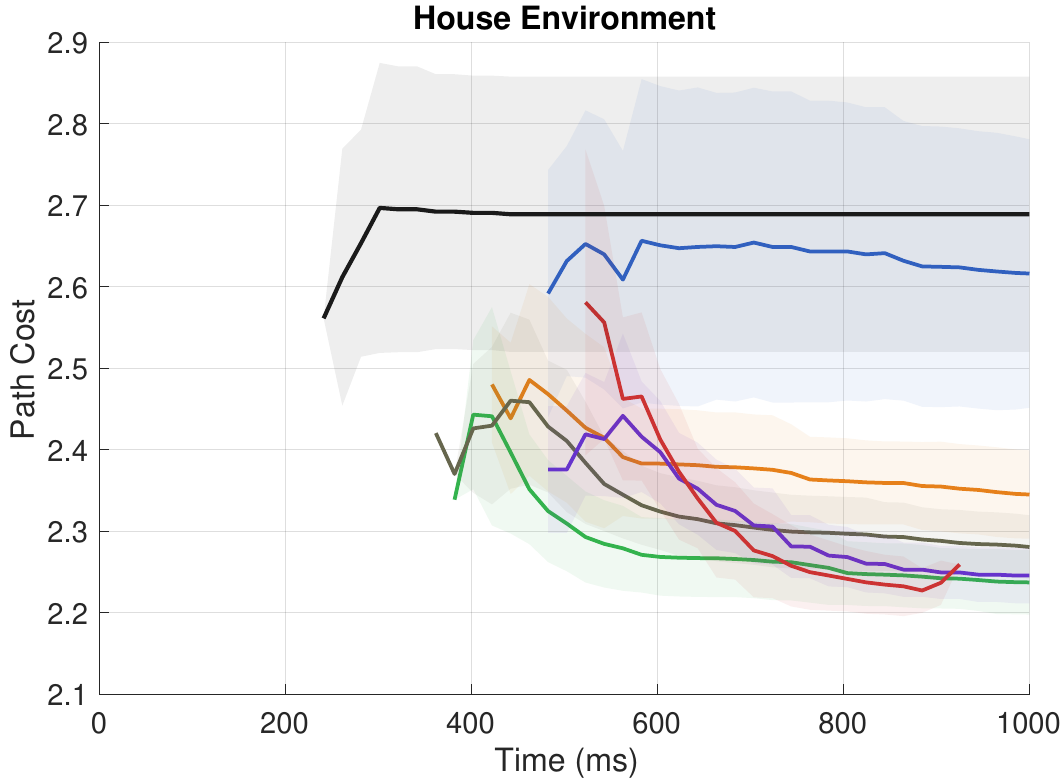}
    \end{subfigure}
    \vspace{-3mm}
    \caption{Benchmarks on Jetson Orin Nano with various $\delta$ sizes.}
    \label{fig: jetson}
\end{figure*}

To assess deployability and characterize sensitivity to $\delta$, we additionally benchmark \aoalg on an embedded Jetson Orin Nano GPU across two of the four environments using the 6D double integrator system. Here, rather than the extreme values of $\delta$ used above, we assess a range of intermediate values.

The results are shown in Fig.~\ref{fig: jetson}. Across all six decomposition ($\delta$) sizes, \aoalg produces a first solution within roughly $300–500 \,\mathrm{ms}$ and continues to reduce path cost over the $1\,\mathrm{s}$ planning budget, in both the Forest and House environments. 
First-solution latency is at most $0.2\,\mathrm{s}$ slower than \alg, with comparable or lower cost, demonstrating that the asymptotic near-optimality of \aoalg comes at minimal overhead to first-solution speed even on resource-constrained hardware.
The decompositions $\delta_2$–$\delta_6$ take similar times to produce an initial solution but keep improving throughout the $1 \,\mathrm{s}$ horizon, converging to a noticeably lower path cost than \alg by the end of planning in both environments. 

More importantly, solution quality is similar across the four intermediate resolutions $\delta_2$–$\delta_5$, with the finest decomposition $\delta_6$ achieving the lowest overall cost but showing greater variance (wider shaded bands) during convergence, reflecting the added planning overhead of a much larger region count. 

These results show that the asymptotic near-optimality of \aoalg holds across a broad range of $\delta$ on embedded hardware. In practice,
any mid-range $\delta$ performs well, i.e., offers a strong quality-to-speed trade-off, and \aoalg does not require careful tuning of this parameter. Finally, and importantly, the near-optimality guarantee of \aoalg costs little in first-solution latency relative to \alg.

%% file: sections/conclusion.tex
We have introduced \aoalg, a novel parallelized asymptotically near-optimal kinodynamic motion planner. Benchmark results show that initial solution times are orders of magnitude faster than those of existing asymptotically near-optimal kinodynamic SBMPs, with significantly lower initial solution costs. Furthermore, it consistently outperforms its predecessor, \alg, in solution costs. For future work, we plan to conduct deployments in real-world robotic systems.

%% file: sections/GPU_implementation.tex
\subsection{Hyperparameters and Implementation Considerations}
\label{appendix: hyperparameters}

In this appendix, we discuss how \aoalg can be tuned to match a problem's difficulty, and provide discussions of \aoalg that enable efficient parallelism. The three aspects are the expected tree size, space decomposition and using atomic operations.

\subsubsection{Expected tree size tuning}

Introduced in \alg by \cite{Perrault2025kinopax}, the expected tree size $t_e$ is a fixed hyperparameter that specifies the maximum number of nodes allowed in $\mathcal{T}$ and sets bounds for the sizes of $V_A$, $V_I$, $V_T$, and $V_U$. Fixing $t_e$ improves the computational efficiency of \aoalg by eliminating high-latency operations associated with dynamic memory resizing on many-core devices. Additionally, a fixed tree size enables users to precisely tune algorithm performance, balancing the trade-off between runtime and solution quality.

Adjusting the value of $t_e$ affects \aoalg's behavior in two ways. First, it influences $\lambda$, which is calculated by:
\begin{equation}
\label{eq:AO_BF}
    \lambda = \left\lfloor\frac{t_e}{|V_A|}\right\rfloor.
\end{equation}
This relation ensures maximal propagation of each active node within a single iteration of the \texttt{Propagate} subroutine while guaranteeing that \texttt{Propagate} does not generate more than $t_e$ new nodes, which would result in a segmentation fault. Increasing $t_e$ leads to the generation of more trajectory segments (Algorithm~\ref{algo:ao_Propagate}), potentially yielding lower-cost trajectories. However, higher $t_e$ values require greater computational resources and additional memory to accommodate larger node sets and trees, ultimately increasing iteration times. This trade-off is examined in detail in experiments in Section~\ref{sec:OKPAX_experiments}.

Secondly, $t_e$ controls the total number of nodes stored within $\mathcal{T}$. For complex problems, particularly those involving high dimensional systems, a larger $t_e$ is recommended to enable sufficient exploration of the state space and identification of high-quality solutions. Due to \aoalg's structured search strategy, hyperparameter tuning is quite straightforward.
This is demonstrated in Section~\ref{sec:OKPAX_experiments}, where we consistently achieve high-quality solutions across a wide range of $t_e$ values and problem configurations.

\subsubsection{Space Decomposition Tuning}

\aoalg allows tuning through its spatial decomposition $\mathcal{R}$, where adjusting the fineness of the decomposition changes algorithm behavior.

With a coarser decomposition, fewer nodes are accepted into $\mathcal{T}$, as multiple nodes are more likely to fall within the same region. This configuration improves memory efficiency by enabling solutions to be found with a relatively small $t_e$. Moreover, a coarser discretization helps prevent \aoalg from exhausting memory prematurely, allowing the algorithm to run for a sufficient number of iterations to iteratively improve the solution. However, if the discretization is too coarse, it can hinder the algorithm’s ability to propagate nodes between regions, potentially preventing the discovery of feasible solutions.

Additionally, a finer decomposition significantly changes the algorithm's behavior. If the decomposition is excessively fine relative to a small $t_e$, newly generated nodes frequently populate unvisited regions, rapidly exhausting available memory before sufficient iterations can be executed. In contrast, pairing a fine decomposition with a large $t_e$ enables high-quality solutions and sufficient iterations but increases per-iteration runtime, delaying the discovery of the first solution. As shown in Section~\ref{sec:OKPAX_experiments}, the spatial decomposition in \aoalg can be tuned effectively, with quality solutions obtained across a wide range of decompositions and environment complexities for a given dynamical system.

\subsubsection{Data Structure Representation}

Unlike traditional SBMPs, which represent the planning tree using node objects, \aoalg represents $\mathcal{T}$ as a one-dimensional floating-point array. Prior to planning, the user specifies the expected tree size $t_e$, determining the GPU memory allocation. Specifically, $\mathcal{T}$ is structured as a floating-point array of size:
\[
\mathcal{T}_{\text{size}} = t_e \cdot (s_{\text{dim}} + c_{\text{dim}} + 2),
\]
where $s_{\text{dim}}$ denotes the state dimension, $c_{\text{dim}}$ the control dimension, and the additional $2$ represents the control duration and parent node index. For instance, a node in $\mathcal{T}$ for a 6D Double Integrator is structured as (Fig.~\ref{fig:T_node_structure}):

\begin{figure}[h]
  \centering
  \includegraphics[width=\linewidth]{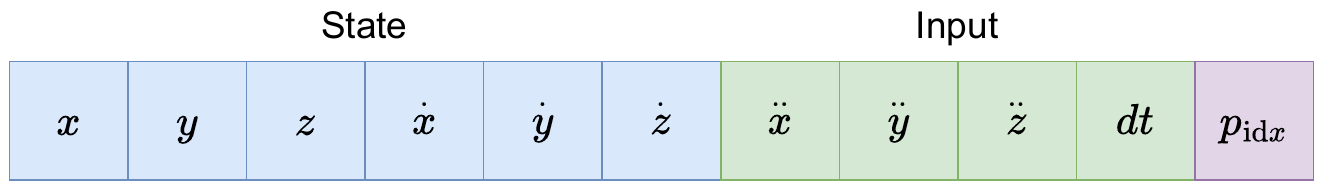}
  \caption{Example node structure in $\mathcal{T}$ for a 6D Double Integrator system.}
  \label{fig:T_node_structure}
\end{figure}

This representation significantly enhances efficiency for parallel computation. Threads independently read states from assigned indices of interest in $\mathcal{T}$, propagate trajectory segments, and safely write results back to predetermined locations. A key challenge arises in efficiently identifying these indices of interest, essentially a parallelized graph search over active nodes.

To facilitate parallel graph searches, node sets in \aoalg are represented as boolean mask arrays on the GPU. These arrays have length $t_e$, where a value of `1' indicates the corresponding node is included in the set. For example, with $t_e = 5$ and nodes $0$ and $3$ active in $V_A$, the representation is as (Fig.~\ref{fig:V_A_boolean_mask}):

\begin{figure}[ht!]
  \centering
  \includegraphics[width=0.6\columnwidth]{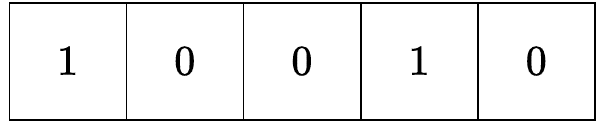}
  \caption{Boolean mask representation for active nodes in $V_A$.}
  \label{fig:V_A_boolean_mask}
\end{figure}

To efficiently perform a graph search on $\mathcal{T}$ and extract active indices (e.g., nodes 0 and 3 in the example), we utilize two GPU-optimized operations. First, an \texttt{exclusiveScan} operation computes the cumulative sum of the $V_A$ mask, storing the results in a separate array, $S$. Our implementation employs optimized parallel operations provided by CUDA's Thrust library~\cite{guide2020cuda}. The result of \texttt{exclusiveScan} for the example is (Fig.~\ref{fig:exclusive_scan_result}):

\begin{figure}[ht]
  \centering
  \includegraphics[width=0.6\columnwidth]{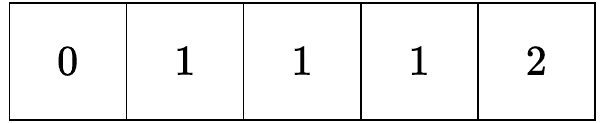}
  \caption{Resulting cumulative sum array $S$ after applying \texttt{exclusiveScan} to $V_A$.}
  \label{fig:exclusive_scan_result}
\end{figure}

Next, arrays $S$ and $V_A$ serve as inputs to a second parallel GPU routine, \texttt{GraphSearch}, which compacts active indices from $V_A$ into a new integer array, $I$, also of size $t_e$. The resulting array $I$ for our example and pseudocode for \texttt{GraphSearch} are provided below (Fig.~\ref{fig:GraphSearch_result}):

\begin{figure}[ht]
  \centering
  \includegraphics[width=0.6\columnwidth]{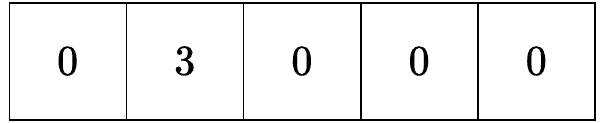}
  \caption{Resulting array $I$ containing active indices after executing \texttt{GraphSearch}.}
  \label{fig:GraphSearch_result}
\end{figure}

\begin{algorithm}[h]
    \caption{\texttt{GraphSearch}}
    \label{algo:GraphSearch}
    \SetKwInOut{Input}{Input}
    \SetKwInOut{Output}{Output}
    \Input{$V_A, S, I$}
    \Output{Updated array $I$}
    $tid \gets $ unique thread index\;
    \If{$tid \geq |V_A|$ \textbf{or} $\neg V_A[tid]$}{
        \Return{}\;
    }
    $I[S[tid]] \gets tid$\;
\end{algorithm}

The \texttt{GraphSearch} subroutine launches $t_e$ threads, each assigned to a unique index in $V_A$. Each thread determines its unique ID, exiting immediately if its corresponding node is inactive (Algorithm~\ref{algo:GraphSearch}, Lines 1--3). Active threads compute their target position in $I$ using $S[tid]$. Since $S$ remains constant until encountering subsequent active nodes, each thread safely writes its unique identifier to a distinct position in $I$ (Algorithm~\ref{algo:GraphSearch}, Line 4). After all threads complete execution, array $I$ compactly represents active node indices, and the final value in array $S$ indicates the total number of active nodes. This compact representation provides essential data for initiating the subsequent parallel propagation phase.

\subsubsection{Atomic Operations}
\label{appendix: atomic operations}

Given our design choices, data structure layout, and maintenance methods, we now detail the computational workflow of \aoalg. Assume an expansion set $V_A$ has undergone the parallel \texttt{exclusiveScan} and \texttt{GraphSearch} operations to populate arrays $S$ and $I$. To initiate the \texttt{Propagate} subroutine Algorithm~\ref{algo:ao_Propagate}, only a single integer (the count of active nodes in $V_A$ obtained from the final index of $S$) must be transferred from the GPU to the CPU. Using this integer, the CPU orchestrates the GPU's \texttt{Propagate} kernel by launching $|V_A| \cdot \lambda$ threads. This ensures the number of threads never exceeds $t_e$, thus preventing memory out-of-bounds errors.

We now describe the parallel \texttt{Propagate} routine in detail for \aoalg (Algorithm~\ref{algo:AOKPAX_P_Prop}). This GPU kernel takes as inputs $\mathcal{T}$, $I$, $O$, $\mathcal{R}$, $\lambda$, $V_U$, and an empty tree structure $\mathcal{U}$ (with identical size to $\mathcal{T}$).

Each thread first calculates its unique identifier ($tid$), ranging from 0 to $|V_A|\cdot\lambda - 1$. Using this identifier, it obtains its assigned index from $I$ by evaluating $I[\lfloor tid/\lambda \rfloor]$. This indexing scheme ensures exactly $\lambda$ threads access each active node, while no threads exceed the compact set $I$. Threads then read their assigned state data $x$ from the array $\mathcal{T}$ using index arithmetic involving the system dimension (e.g., for the 6D double integrator from Fig.~\ref{fig:T_node_structure}, $\text{DIM} = 11$) (Algorithm~\ref{algo:AOKPAX_P_Prop}, Lines 1--3).

After obtaining the node $x$, each thread randomly generates controls and durations, propagating $x$ to a new state $x'$ (Algorithm~\ref{algo:AOKPAX_P_Prop}, Lines 4--5). Threads then verify trajectory validity, determine the corresponding region $\mathcal{R}_i$ for state $x'$, and calculate the cost to reach $x'$ (Lines 6--8).

To prevent race conditions, threads use atomic operations to update the minimum cost of $\mathcal{R}_i$ (Line 9). This atomic operation ensures that if multiple threads attempt to write to the memory location of $\mathcal{R}_i$ simultaneously there is an ordering constraint, preventing data races. If a thread's node $x'$ represents the minimal cost within its region, the thread writes $x'$ into the empty tree structure $\mathcal{U}$ at index $tid \cdot \text{DIM}$, guaranteeing unique write locations for each thread. The thread then marks the corresponding position in the boolean mask array $V_U$ to indicate the addition of a promising node (Lines 10--12). Lastly if a thread's node $x'$ is in the goal region and has a lower cost than the best-known goal satisfying trajectory $\mathbf{x}$ it atomically updates the best-known solution cost (Lines 13--14).

The rationale for using $\mathcal{U}$ rather than directly writing to $\mathcal{T}$ is threefold: 
(i) direct writing to $\mathcal{T}$ requires knowledge of empty indices of $\mathcal{T}$ to avoid overwriting critical data,
(ii) writing directly to a partially populated $\mathcal{T}$ would produce sparse, inefficient memory usage, and 
(iii) direct writes limit parallelism since at most $t_e - |\mathcal{T}|$ threads could safely run without overwriting. Using a separate, initially empty structure $\mathcal{U}$ allows launching all $t_e$ threads in every iteration, with a subsequent compact transfer of new nodes to $\mathcal{T}$.

The complete pseudocode for the parallel propagation routine in \aoalg is provided below:

\begin{algorithm}
    \caption{\texttt{\aoalg Parallel Propagate}}
    \label{algo:AOKPAX_P_Prop}
    \SetKwInOut{Input}{Input}
    \SetKwInOut{Output}{Output}
    \Input{$\mathcal{T}, I, O, \mathcal{R}, \lambda, \mathcal{U}, V_U$}
    \Output{Updated arrays $\mathcal{U}$ and $V_U$}
    $tid \gets$ unique thread index\;
    $treeIdx \gets I[\lfloor tid / \lambda \rfloor]$\;
    $x \gets \mathcal{T}[treeIdx \cdot DIM : (treeIdx+1)\cdot DIM]$\;
    Randomly sample control $u$ and duration $dt$\;
    $x' \gets \rv{\texttt{PropagateODE}}(x, u, dt)$\;
    \If{trajectory from $x$ to $x'$ is valid}{
        Map $x'$ to region $\mathcal{R}_i$\;
        $cost(x') \gets cost(x) + \texttt{getCost}(x,x')$\;
        $cost(\mathcal{R}_i) \gets \texttt{atomicMin}(cost(\mathcal{R}_i), cost(x'))$\;
        \If{$cost(x') = cost(\mathcal{R}_i)$}{
            $\mathcal{U}[tid \cdot DIM : (tid+1)\cdot DIM] \gets x'$\;
            $V_U[tid] \gets 1$\;
            \If{$x' \in X_{goal}$}{
                $cost(\mathbf{x}) \gets \texttt{atomicMin}(cost(\mathbf{x}'), cost(x'))$\;
            }
        }
    }
\end{algorithm}

After completing the \texttt{Propagate} routine, the unexplored tree structure $\mathcal{U}$ and boolean mask array $V_U$ are populated, and the minimum costs for visited regions in $\mathcal{R}$ are updated. With these updated region costs, we proceed to prune nodes from the tree $\mathcal{T}$. The GPU implementation of node pruning closely follows the serialized pseudocode (Algorithm~\ref{algo:ao_NodePruning}); no atomic operations are necessary since each GPU thread independently inspects a single node at index $tid \cdot DIM$ and updates the corresponding boolean value in $V_A[tid]$ to indicate whether the node remains active or is pruned.

However, a key implementation detail involves filtering nodes added to $\mathcal{U}$ that no longer represent the lowest-cost node within their respective regions. This situation arises when one thread initially adds a node as the region's lowest-cost node, only for a subsequent thread (within the same subprocess) to identify a node with lower cost in the same region. To resolve this, we perform an additional GPU subprocess using our graph search routine (Algorithm~\ref{algo:GraphSearch}) to update array $I$ with active indices from $V_U$. We then launch $|V_U|$ threads to reevaluate and filter out these higher-cost nodes, ensuring that only the lowest-cost nodes remain marked as active in $V_U$.

Next, we detail the GPU implementation of the \texttt{UpdateTree} subroutine. This step efficiently transfers new nodes from the temporary tree $\mathcal{U}$ to the main tree $\mathcal{T}$ and identifies if an improved solution trajectory has been found. First, we perform another graph search on $\mathcal{U}$ over the boolean set $V_U$, identifying active indices. Then, we launch a GPU thread for each node in $\mathcal{U}$ and transfer it to a compact position in $\mathcal{T}$ at index $|\mathcal{T}| \cdot DIM + tid \cdot DIM$. This ensures efficient, concurrent, and safe writing without overlapping indices. Lastly, threads check if their newly transferred node is within the goal region and has cost equal to that updated in Algorithm~\ref{algo:AOKPAX_P_Prop}. If so, the thread updates a solution integer variable with the tree index of this node, thus efficiently maintaining the optimal solution trajectory.

We then repeat the three core subprocesses: \texttt{Node Propagate}, \texttt{Node Pruning}, and \texttt{Tree Update}. Critically, we designed this cycle to minimize latency between subprocesses by reducing both the frequency and volume of data transfers between the CPU and GPU. Figure~\ref{fig:implementation} illustrates this repeated cycle and highlights the latency periods.

\begin{figure}[h!]
    \centering
    \includegraphics[width=\linewidth]{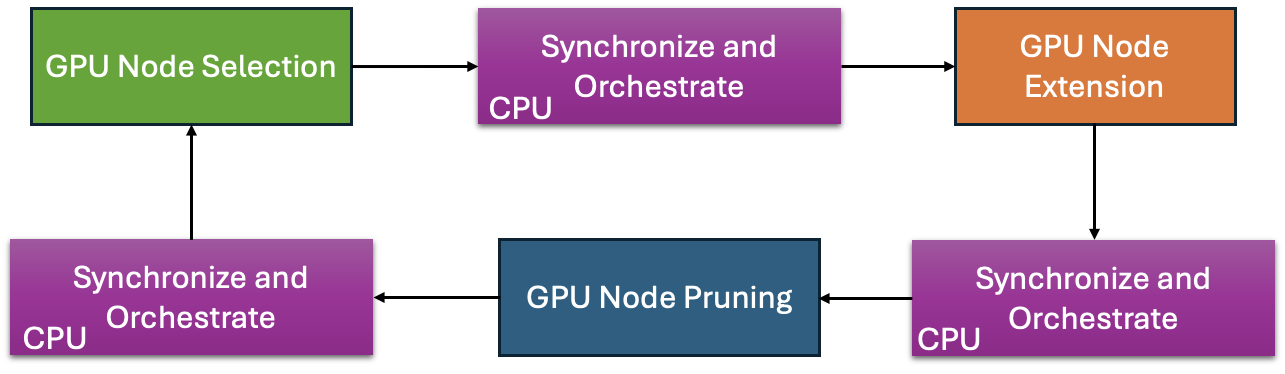}
    \caption{Planning cycle of \aoalg, consisting of repeated GPU subprocesses. A CPU-side synchronization occurs between each subprocess to coordinate data exchange and maintain execution order.}
    \label{fig:implementation}
\end{figure}

\begin{remark}
If collision checking requires significant computational resources, particularly in obstacle-dense environments, performance can be improved by first comparing the cost to reach \( x' \) against the lowest recorded cost for region \( \mathcal{R}_i \). Collision checking is then performed only if \( x' \) is the best recorded node in \( \mathcal{R}_i \).
\end{remark}

\subsection{Empirical Evaluation Details}
\label{appendix: empirical eval details}

The two variants of \aoalg are \aoalg{}\textit{-large-$\delta$} and \aoalg{}\textit{-small-$\delta$}. These configurations vary the decomposition size $\delta$ and the pre-allocated expected tree size $t_e$. The \aoalg{}\textit{-large-$\delta$} strategy uses a coarser decomposition and smaller expected tree sizes, emphasizing rapid first solutions and early termination. Specifically, for the 6D double integrator: 
$t_e= 10^5$ with 
$27,000$ regions; for the 6D Dubins airplane: 
$t_e=3 \times 10^5$ with 
52,000 regions; and for the 12D nonlinear quadcopter: 
$t_e=10^6$ with 
$10^5$ regions. Conversely, the \aoalg{}\textit{-small-$\delta$} strategy uses 
a finer decomposition and larger tree sizes ($t_e=10^8$ with 
$10^7$ regions for all systems.
), focusing on extensive state space exploration.